\documentclass[letterpaper]{article} 
\usepackage{aaai2026}  
\usepackage{times}  
\usepackage{helvet}  
\usepackage{courier}  
\usepackage[hyphens]{url}  
\usepackage{graphicx} 
\usepackage{natbib}  
\usepackage{caption} 
\usepackage{algorithm}
\usepackage{algorithmic}

\usepackage{newfloat}
\usepackage{listings}
\DeclareCaptionStyle{ruled}{labelfont=normalfont,labelsep=colon,strut=off} 
\floatstyle{ruled}
\newfloat{listing}{tb}{lst}{}
\floatname{listing}{Listing}
\usepackage{amssymb}
\usepackage{bm}
\usepackage{graphicx}
\usepackage{multirow} 

\usepackage{amsthm}
\usepackage{subfigure}
\usepackage{amsmath}
\usepackage{dsfont} 
\usepackage{booktabs}

\newtheorem{lemma}{Lemma}
\newtheorem{proposition}{Proposition}
\newtheorem{definition}{Definition} 

\nocopyright 

\title{Dual-Kernel Graph Community Contrastive Learning}
\author{
    Xiang Chen\textsuperscript{\rm 1,\rm 2}, 
    Kun Yue\textsuperscript{\rm 1,\rm 2}, 
    Wenjie Liu\textsuperscript{\rm 1,\rm 2}, 
    Zhenyu Zhang\textsuperscript{\rm 3}, 
    Liang Duan\textsuperscript{\rm 1,\rm 2}\thanks{Corresponding author.}
}
\affiliations{
    \textsuperscript{\rm 1}School of Information Science and Engineering, Yunnan University, Kunming, China\\
    \textsuperscript{\rm 2}Yunnan Key Laboratory of Intelligent Systems and Computing, Yunnan University, Kunming, China\\
    \textsuperscript{\rm 3}College of Artificial Intelligence, Tianjin University of Science and Technology, Tianjin, China\\
    chenx@stu.ynu.edu.cn, duanl@ynu.edu.cn
}

\usepackage{bibentry}

\begin{document}

\maketitle

\begin{abstract}
    Graph Contrastive Learning (GCL) has emerged as a powerful paradigm for training Graph Neural Networks (GNNs) in the absence of task-specific labels. However, its scalability on large-scale graphs is hindered by the intensive message passing mechanism of GNN and the quadratic computational complexity of contrastive loss over positive and negative node pairs. To address these issues, we propose an efficient GCL framework that transforms the input graph into a compact network of interconnected node sets while preserving structural information across communities. We firstly introduce a kernelized graph community contrastive loss with linear complexity, enabling effective information transfer among node sets to capture hierarchical structural information of the graph. We then incorporate a knowledge distillation technique into the decoupled GNN architecture to accelerate inference while maintaining strong generalization performance. Extensive experiments on sixteen real-world datasets of varying scales demonstrate that our method outperforms state-of-the-art GCL baselines in both effectiveness and scalability.
\end{abstract}

\begin{links}
    \link{Code}{https://github.com/chenx-hi/DKGCCL}
\end{links}

\section{Introduction}
Graph neural networks (GNNs) learn effective node representations through message passing over graph structures, achieving impressive success across a wide range of graph analysis tasks~\cite{gnn-survey}. However, most GNN models are trained in a supervised way, and their performance is heavily dependent on the availability of labeled data. To address this limitation, graph contrastive learning (GCL) has emerged as a promising self-supervised approach for graph representation learning~\cite{dgi}. The core idea of GCL is to distinguish positive and negative node pairs using a contrastive loss grounded in mutual information maximization. This enables label-free training of GNN, achieving performance comparable to, or even surpassing, that of supervised methods~\cite{gcl-survey}.


Despite significant progress in GCL, its application to large-scale graphs remains challenging due to two fundamental bottlenecks. The first is that \textit{training scalability is limited by the quadratic computational complexity of pairwise comparisons in both GNN and contrastive loss}. The second is that \textit{inference inefficiency stems from the intensive message-passing mechanism inherent in GNN architectures}. Most existing methods address these challenges in isolation, lacking a unified framework that improves both scalability and efficiency. For instance, some methods simplify training by eliminating the need for negative sampling~\cite{e2neg} or reducing the number of augmented views processed by GNN~\cite{sugrl}, but they do not improve inference efficiency. Conversely, other methods achieve fast inference through decoupled GNN-MLP architectures~\cite{graphecl}, yet still incur high computational training cost due to the contrastive loss.

Recent advancements in graph representation learning have extended beyond the limitations of traditional node-level message passing to reduce the complexity of GNN during training. These methods typically partition the graph into multiple communities, treating each as a single node to yield a coarsened graph. Subsequently, message passing~\cite{cluster-gcn} or Graph Transformer attention~\cite{smooth} is applied to the coarsened graph, enabling the model to capture long-range dependencies and significantly reduce computational complexity. Notably, this graph coarsening technique has been successfully applied to GCL~\cite{structcomp}, where it simultaneously addresses scalability issues arising from both GNN and contrastive loss in training phase. However, such oversimplification may lead to excessively uniform node representations within communities, resulting in a loss of fine-grained node information.

In this work, instead of simplifying each community to a single node via coarsening techniques, we envision the input graph as a network of node sets interconnected across communities, which allows us to preserve essential node information during training~\cite{n2c}. To mitigate the increased complexity arising from this design choice, we integrate Multiple Kernel Learning (MKL)~\cite{mkl-tkde, mkl-federated} into the graph contrastive loss. By combining node-level and community-level kernels of different granularities, our proposed Dual-Kernel \underline{G}raph \underline{C}ommunity \underline{C}ontrastive \underline{L}earning (GCCL) effectively captures the hierarchical structural information of the graph~\cite{ais}. In the GCCL training process, we forgo explicit message passing and reduce the computational complexity of the contrastive loss from quadratic to linear time, thereby directly addressing the scalability challenge.


To improve inference efficiency on large-scale graphs while guaranteeing model performance, we propose a knowledge distillation module based on a decoupled GNN architecture. Specifically, we decouple the feature transformation and message passing steps of the GNN. The linear layer is trained within GCCL, while during inference, a parameter-free message passing operation is employed to propagate structural information  across the graph. This decoupled paradigm introduces no additional training overhead  on GCCL and improves generalization for downstream tasks. We then use the post-message-passing representations from the decoupled GNN as the distillation target,  which enables us to extract a lightweight MLP model that captures graph structural information from the decoupled GNN, making our method suitable for latency-critical applications.

Our main contributions are summarized as follows:
\begin{itemize}
    \item We propose a dual-kernel graph community contrastive loss by integrating multiple kernel learning, which improves the scalability of GCL training.
    
    \item We introduce a knowledge distillation module for decoupled GNN to effectively preserve graph structure information and enable low-latency inference.
    
    \item We provide theoretical analyses demonstrating that the kernelized graph community contrastive loss yields high-quality node representations for downstream tasks.
      
    \item Extensive experiments show that our method achieves state-of-the-art performance while significantly reducing the computational costs of both training and inference.
\end{itemize}

\section{Preliminaries}
\subsubsection{Graph Neural Network.}
Let $\mathcal{G} = (\mathcal{V}, \mathcal{E}, \mathbf{X})$ denote an undirected graph, where $\mathcal{V} = \{v_1, \cdots, v_n\}$ is the set of $n$ nodes, $\mathcal{E} \subseteq \mathcal{V} \times \mathcal{V} $ is the set of edges, and $\mathbf{X} \in \mathbb{R}^{n \times h}$ is the node feature matrix. The $i$-th row of $\mathbf{X}$ corresponds to the $h$-dimensional feature vector $\mathbf{x}_i$ of node $v_i$. The graph structure can be denoted by an adjacency matrix $\mathbf{A} \in \{0, 1\}^{n \times n}$, where $\mathbf{A}_{i,j} = 1$ if and only if $(v_i,v_j) \in \mathcal{E}$. For simplicity, the undirected graph $\mathcal{G}$ can also be denoted as $G=(\mathbf{A}, \mathbf{X})$. 
Given $G$ as input, the GNN encoder $f_{\theta}(G):G \to \mathbb{R}^{ n \times d}$ can produce effective representations $\mathbf{z}_i = f_{\theta}(G)[v_i]$ of node $ v_i$:
\begin{equation} \label{equ:gnn}
	f_{\theta}(G) = \sigma(\widetilde{\mathbf{D}}^{-\frac{1}{2}} (\mathbf{A} + \mathbf{I}) \widetilde{\mathbf{D}}^{-\frac{1}{2}} \mathbf{X}\mathbf{W})
\end{equation}
where $\mathbf{I}$ is an identity matrix, $\widetilde{\mathbf{D}}$ is a diagonal degree matrix of $\mathbf{A} + \mathbf{I}$, $\sigma(\cdot)$ is a non-linear activation function, and $\mathbf{W} \in \mathbb{R}^{h \times d}$ is a learnable parameter matrix corresponding to $\theta$. 

\subsubsection{Graph Community Contrastive Learning.} A typical GCL paradigm defines adjacent nodes as positive pairs and all other nodes in $G$ as negative pairs to ensure the adjacent nodes have similar representations~\cite{ncla}:
\begin{equation} \label{equ:gcl}
    \mathcal{L}_{G} = - \frac{1}{n} \sum_{v_i \in \mathcal{V}} \frac{1}{\mathcal{N}(v_i)} \sum_{v_j \in \mathcal{N}(v_i)} \log {\frac{\exp{(\mathbf{z}_i^\mathrm{T}\mathbf{z}_j/\tau)}}{\sum_{v_k \in v_i^-} \exp(\mathbf{z}_i^\mathrm{T}\mathbf{z}_k/\tau)}}
\end{equation}
where $\mathcal{N}(v_i)$ is the neighborhoods of node $v_i$, $v_i^-$ is the negative node set and $\tau$ is the temperature hyper-parameter.
 
Let $\mathcal{P} = \{P_1, \cdots, P_m\}$ be a partition of $G$ with $m$ communities. Each community $P_j \in \mathcal{P}$ is a subset of $\mathcal{V}$, such that $\mathcal{V}=\bigcup_{j=1}^m P_j$ and $P_j \cap P_k = \emptyset$ for $j \neq k$. The partition assignment matrix denotes as $\mathbf{P} \in \mathbb{R}^{n \times m}$, where $\mathbf{P}_{i,j}$ is the weight of $i$-th node in the $j$-th community. The community-wise (coarsened) graph $\mathbf{A}^\mathcal{P}$ can be constructed by $\mathbf{P}^\mathrm{T} \mathbf{A} \mathbf{P}$, where $\mathbf{A}^\mathcal{P}_{j,k}$ is the connection weight between $P_j$ and $P_k$. 

In this work, we focus on leveraging community structure information to reconstruct the graph contrastive loss in Eq.~\ref{equ:gcl}, and refer to the resulting objective as the graph community contrastive loss (GCCL loss).

\subsubsection{Multiple Kernel Learning.} Kernel-based methods utilize an optimal kernel function $\kappa(\cdot,\cdot): \mathbb{R}^d \times \mathbb{R}^d \to \mathbb{R}$ to measure pairwise similarity and have proven powerful for diverse tasks~\cite{mkl-tkde}. MKL methods integrate diverse features from different perspectives by combining multiple kernel functions~\cite{mkl-cluster}. The resulting multiple kernel $\kappa_\eta$ is defined as:
\begin{equation} \label{equ:mkl}
    \kappa_\eta(\{\mathbf{z}_i\}^l_{i=1}, \{\mathbf{y}_i\}^l_{i=1}) = g_\eta(\{ \kappa_i(\mathbf{z}_i, \mathbf{y}_i) \}^l_{i=1}) 
\end{equation}
where $g_\eta: \mathbb{R}^l \to \mathbb{R}$ is a combinatorial function. In this work, we focus on the pairwise scenario, i.e., $l=2$. Specially, given two kernels $\kappa_1: \mathcal{X} \times \mathcal{X} \to \mathbb{R}$ and $\kappa_2: \mathcal{X}^\prime \times \mathcal{X}^\prime \to \mathbb{R}$, we consider two strategies of $g_\eta$: the tensor product of kernels and the convex linear of kernels~\cite{mkl-two}. 
The tensor product method is defined as:
\begin{equation} \label{equ:mkl-dot}
    \kappa_\eta((\mathbf{z},\mathbf{z}^\prime), (\mathbf{y},\mathbf{y}^\prime)) = \kappa_1(\mathbf{z},\mathbf{y}) \cdot \kappa_2(\mathbf{z}^\prime, \mathbf{y}^\prime)
\end{equation}
where $(\mathbf{z},\mathbf{z}^\prime), (\mathbf{y},\mathbf{y}^\prime)$ are pairs of objects from feature space $\mathcal{X} \times \mathcal{X}^\prime$. The convex linear combination is defined as:
\begin{equation} \label{equ:mkl-combination}
    \kappa_\eta((\mathbf{z},\mathbf{z}^\prime), (\mathbf{y},\mathbf{y}^\prime)) = \alpha\kappa_1(\mathbf{z},\mathbf{y}) + (1-\alpha) \kappa_2(\mathbf{z}^\prime, \mathbf{y}^\prime)
\end{equation}
where $\alpha \in [0, 1]$ is a combination coefficient.

\begin{figure*}[t]
\centering
\includegraphics[width=0.99\textwidth]{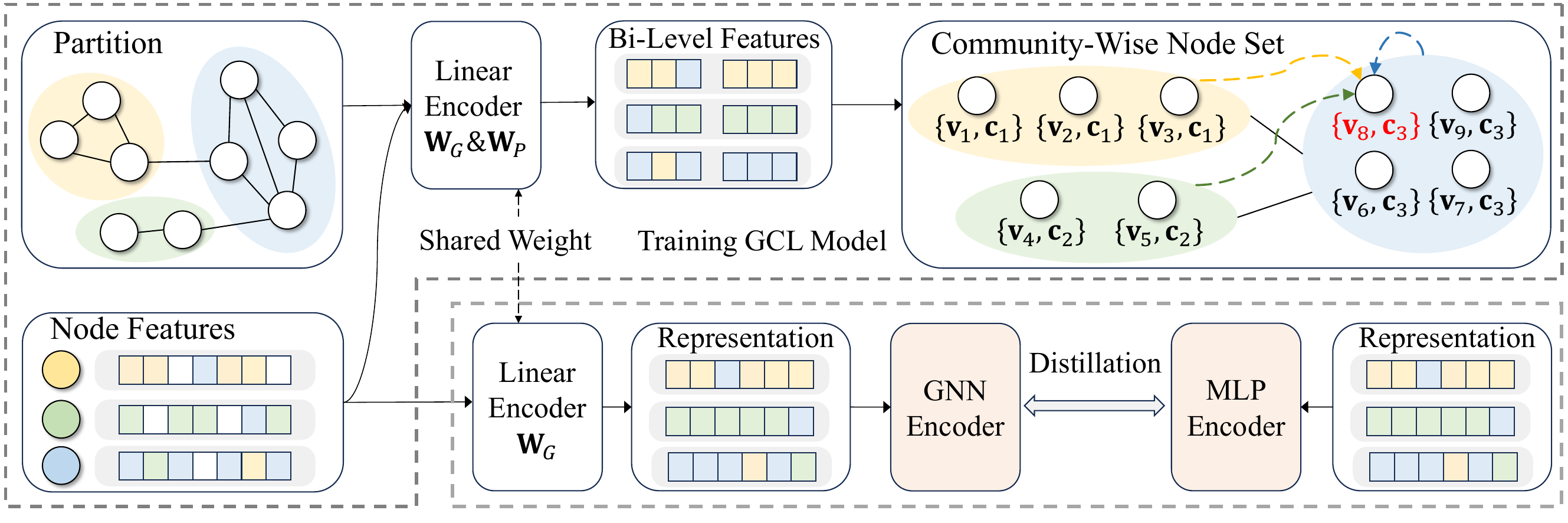}
\caption{The overall framework of our method.}
\label{fig-model}
\end{figure*}

\section{Methodology}
In this section, we present the proposed GCCL framework, illustrated in Figure~\ref{fig-model}. We first reduce the complexity of contrastive loss by leveraging community structure and MKL. We then design a knowledge distillation module for decoupled GNN to speedup inference.



\subsection{Kernel-Based GCCL}
\subsubsection{Bi-Level Features Generation.}~
The fundamental concept behind our method is to perceive $G$ as interconnected communities of nodes~\cite{n2c}. Within this paradigm, the community to which a node set belongs can serve as a bridge for information interaction, enhancing the information flow between its internal nodes and nodes in other communities. This enables our method to effectively capture the hierarchical information of the graph while preserving node-level details. Thus, node $v_i$ in $P_j$ can be characterized by a bi-level pair of features $\{\mathbf{v}_i, \mathbf{c}_j\} \in \mathcal{X}_G \times \mathcal{X}_P$:
\begin{equation}  \label{equ:bi-level}
    \mathbf{v}_i = \mathbf{x}_i\mathbf{W}_G,~\mathbf{c}_j= \sum\nolimits_{v_t \in P_j} \mathbf{P}^{\mathrm{T}}_{t,j} \text{Dropout}\left(\mathbf{x}_t\mathbf{W}_P\right)
\end{equation}
where $\mathbf{v}_i \in \mathcal{X}_G$ is a node-level feature and $\mathbf{c}_j \in \mathcal{X}_P$ is a community-level feature. $\mathbf{W}_G$ and $\mathbf{W}_P$ are two different projection matrices to the node feature space $\mathcal{X}_G$ and the community feature space $\mathcal{X}_P$, respectively. 

Note that we apply a random mask to all dimensions of $\mathbf{x}_t\mathbf{W}_P$ per training epoch to obtain $\mathbf{c}_j$. This method serves as a special data augmentation strategy, providing more diverse community-level features in $\mathcal{X}_P$ for community structure-based GCL. Specifically, we regard the construction of community-level features as a message passing process from nodes to community centroids. Based on the findings of Dropout($\cdot$) in the message passing mechanism of GNN~\cite{dropmessage,dropout-gnn}, we can derive the following proposition.
\begin{proposition} \label{proposition:drop-p}
    Let the feature dimension of the community-level feature space $\mathcal{X}_P$ be $d^P$. Then, the expected number of distinct partitioned substructures generated by the Dropout($\cdot$) operation for each partition $P_j$ is:
    \begin{equation}
        \mathbb{E}[|P^s_{j}|s=1,\cdots,d^P|] = d^P\left( 1-(1-p)^{|P_j|}\right)
    \end{equation}
    where $P^s_{j}$ is a substructure of $P_j$ on the feature dimension $s$, and $p$ is the dropout probability.
\end{proposition}

Proposition~\ref{proposition:drop-p} demonstrates that Dropout($\cdot$) generates a set of substructures~\cite{dropout-gnn}, whose quantity increases with both the dropout probability $p$ and the dimension $d^P$. In subsequent experiments, we found that the diversity of substructures can reduce the training cycle of GCL. In addition, we will discuss the differences between Dropout($\cdot$) and other augmentation strategies of community-based GCL in Appendix C.1.

\subsubsection{Dual-Kernel GCCL Loss.}
After obtaining the bi-level features, we consider how to use a simple kernel trick to accelerate the computational process of GCL. The success of existing GCL methods lies in emphasizing the neighborhood similarity of node representations~\cite{ncla}, a phenomenon also observed in coarsened graphs~\cite{structcomp}, which aligns with the graph homophily assumption. This motivates us to treat the target node $v_i$ and its interconnected node communities as positive pairs.
\begin{definition}
Given a bi-level kernel $\kappa_B:\mathcal{X}_G \times \mathcal{X}_P \to \mathbb{R}_+$, the graph community contrastive loss can be expressed as:
\begin{equation*}
    \mathcal{L}_\mathcal{P} = -\frac{1}{n} \sum\nolimits_{v_i \in \mathcal{V}} \log  \ell(v_i, P_j) \text{,}~~~~\text{where}~\ell(v_i, P_j) =   
\end{equation*}
\begin{equation*}
     \frac{\sum_{P_k \in \mathcal{N}(P_j)} \mathbf{A}^\mathcal{P}_{j,k} \sum_{v_t \in P_k} \mathbf{P}_{t,k} \cdot \kappa_{B}( \{\mathbf{v}_i, \mathbf{c}_j\}, \{\mathbf{v}_t, \mathbf{c}_k \} ) }{ \sum_{P_k \in \mathcal{P}} \sum_{v_t \in P_k} \mathbf{P}_{t,k} \cdot \kappa_B(\{\mathbf{v}_i, \mathbf{c}_j\}, \{\mathbf{v}_t, \mathbf{c}_k \}) } 
\end{equation*}
where $P_j$ is the community to which $v_i$ belongs, and $\mathcal{N}(P_j)$ is the set of communities connected to $P_j$ in $\mathbf{A}^\mathcal{P}$.
\end{definition}
 In this definition, $\kappa_B$ helps integrate information between nodes and communities, while $\mathbf{A}^\mathcal{P}_{j,k}$ can adjust the weight of positive pair based on the connectivity between $P_j$ and $P_k$.

 Let the non-negative kernel functions of feature spaces $\mathcal{X}_G$ and $\mathcal{X}_P$ be $\kappa_G$ and $\kappa_P$, respectively, and we represent $\kappa_G$ via its feature map as $\kappa_G(\mathbf{v}_i, \mathbf{v}_t) \thickapprox \phi(\mathbf{v}_i)^\mathrm{T}\phi(\mathbf{v}_t)$. According to the tensor product of kernels (Eq.\ref{equ:mkl-dot}), we have:
    \begin{equation*}
        \kappa_B(\{\mathbf{v}_i, \mathbf{c}_j\}, \{\mathbf{v}_t, \mathbf{c}_k \}) = \kappa_P(\mathbf{c}_j, \mathbf{c}_k)\cdot\phi(\mathbf{v}_i)^\mathrm{T}\phi(\mathbf{v}_t)
    \end{equation*}
Then, we can derive the variant of $\ell(v_i, P_j)$ as $\ell_{tp}(v_i, P_j)$.
\begin{definition}
    The dual-kernel GCCL loss with tensor product method $\ell_{tp}(v_i, P_j)$ can be formulated as:
    \begin{equation} \label{eq:tp}
          \frac{\phi(\mathbf{v}_i)^\mathrm{T} \left[\sum_{P_k \in \mathcal{N}(P_j)} \kappa_P^\prime(\mathbf{c}_j, \mathbf{c}_k) \sum_{v_t \in P_k} \phi^\prime(\mathbf{v}_t)\right]}{ \phi(\mathbf{v}_i)^\mathrm{T} \left[ \sum_{P_k \in \mathcal{P}} \kappa_P(\mathbf{c}_j, \mathbf{c}_k) \sum_{v_t \in P_k} \phi^\prime(\mathbf{v}_t) \right] }
    \end{equation}
    where the valid kernel $\kappa_P^\prime(\mathbf{c}_j, \mathbf{c}_k) = \mathbf{A}_{j,k}^\mathcal{P} \cdot \kappa_P(\mathbf{c}_j, \mathbf{c}_k)$ and the feature map $\phi^\prime(\mathbf{v}_t) = \mathbf{P}_{t,k} \cdot \phi(\mathbf{v}_t)$.
\end{definition}

The tensor product method of MKL enables interactions across all dimensions of features at different granularity levels~\cite{mkl-two}, which naturally allows us to effectively capture the dependencies between node-level and community-level features in the combined feature space $\mathcal{X}_G \times \mathcal{X}_P$. Another key advantage of $\ell_{tp}(v_i, P_j)$ is that the summation term of negative pairs in the denominator is shared across all nodes, so it only needs to be calculated once and can be re-used for other nodes. The summation term of positive pairs in the numerator is shared among nodes within the same community. These properties avoid the quadratic computational complexity of node pairs in vanilla contrastive loss.

Next, we discuss another variant of $\ell(v_i, P_j)$. According to the linear combination of kernels (Eq.\ref{equ:mkl-combination}), we have:
\begin{equation*}
    \kappa_B(\{\mathbf{v}_i, \mathbf{c}_j\}, \{\mathbf{v}_t, \mathbf{c}_k \}) = \alpha \phi(\mathbf{v}_i)^\mathrm{T}\phi(\mathbf{v}_t) + (1-\alpha) \kappa_P(\mathbf{c}_j, \mathbf{c}_k)
\end{equation*}
Then, we can derive the variant of $\ell(v_i,P_j)$ as $\ell_{lc}(v_i,P_j)$.

\begin{definition}
    The dual-kernel GCCL loss with linear combination method $\ell_{lc}(v_i, P_j)$ can be formulated as:
    \begin{equation} \label{eq:lc}
          \frac{\phi(\mathbf{v}_i)^\mathrm{T} \left[\sum_{P_k \in \mathcal{N}(P_j)} \left(\sum_{v_t \in P_k} \alpha\phi^{\prime\prime}(\mathbf{v}_t)+\beta \kappa_P^\prime(\mathbf{c}_j, \mathbf{c}_k)\right) \right]}{\phi(\mathbf{v}_i)^\mathrm{T} \left[\alpha\sum_{P_k \in \mathcal{P}} \sum_{v_t \in P_k}\phi^{\prime}(\mathbf{v}_t)+\beta\sum_{P_k \in \mathcal{P}} \kappa_P(\mathbf{c}_j, \mathbf{c}_k) \right]}
    \end{equation}
    where the feature map $\phi^{\prime\prime}(\mathbf{v}_t) = \mathbf{A}_{j,k}^\mathcal{P} \cdot \phi^{\prime}(\mathbf{v}_t)$ and $\beta = 1-\alpha$.
\end{definition}

The convex linear combination of MKL provides the flexibility to combine the effects of features at different granularity levels~\cite{mkl-two}, allowing us to adjust the contribution of node-level and community-level information to the similarity metric of sample pairs via $\alpha$. Similarly, the summation terms in the numerator and denominator of $\ell_{lc}(v_i,P_j)$ are shared  among nodes within the same community and all nodes, respectively. Such a property enables our method to operate on large-scale graphs with fewer computational resources. We will discuss the applicability of variants $\ell_{tp}(v_i,P_j)$ and $\ell_{lc}(v_i,P_j)$ on different datasets in the experimental section and Appendix B.1.

In practice, we employ the simple graph partition algorithm Metis~\cite{metis} to generate $\mathcal{P}$ and ensure the training efficiency. For the feature map $\phi(\mathbf{v})$, we use the sigmoid function to ensure that the similarity in $\mathcal{X}_G$ remains positive. The commonly used  exponential dot product $\exp(\mathbf{c}_j^\mathrm{T}\mathbf{c}_k/\tau)$ is adopted as $\kappa_P$. In Appendix C.2, we further illustrate that the stability of our method is better than other common kernel-based linear similarity measures.

\subsection{Efficient Model Inference}
\subsubsection{Decoupled GNN Architecture.}~The Over-smoothing problem presents a critical challenge hindering the expressive power of GNN~\cite{s3gcl,smooth}. Here, we investigate the impact of our community contrastive loss on node smoothness and illustrate the necessity of incorporating prior information about $G$. Without loss of generality, we take the node classification task as an example. In this task, each node is associated with a label for classification. 

\begin{proposition} \label{proposition:soomth}
    Let $\Bar{\mathbf{A}}$ be the normalized adjacency matrix constructed from positive node pairs in the contrastive loss $\mathcal{L}_{\mathcal{P}}$ and $y(v)$ denote the label of $v$. Then, the bound of the smoothness between node embeddings is:
    \begin{equation}
    \left\|\mathbf{V} - \Bar{\mathbf{A}}\mathbf{V} \right\|_F \leq \sqrt{2}L \sum_{v_i \in \mathcal{V}}  \frac{1}{1 + \frac{\varepsilon(v_i)\lambda_{v_i}}{(1-\varepsilon(v_i))\gamma_{v_i}}}
    \end{equation}
    where $\lambda_{v_i} =  \mathbb{E}_{v_t \in \mathcal{N}(v_i), y(v_i) = y(v_t)} \kappa_{B}( \{\mathbf{v}_i, \mathbf{c}_j\}, \{\mathbf{v}_t, \mathbf{c}_k \} )$ and $\gamma_{v_i} = \mathbb{E}_{v_t \in \mathcal{N}(v_i), y(v_i)\neq y(v_t)} \kappa_{B}( \{\mathbf{v}_i, \mathbf{c}_j\}, \{\mathbf{v}_t, \mathbf{c}_k \} )$. $L$ is the Lipschitz constant, and $\varepsilon(v_i)$ is the one-hop homophily score of node $v_i$ in $\Bar{\mathbf{A}}$, defined as:
    \begin{equation}
        \varepsilon(v_i) = \frac{1}{|\mathcal{N}(v_i)|} \sum\nolimits_{v_t \in \mathcal{N}(v_i)} \mathds{1}[y(v_i)=y(v_t)] 
    \end{equation}
    where $\mathds{1}[\cdot]$ is the indicator function.
\end{proposition}

This proposition establishes a significant relationship between the smoothness of node embeddings and two key factors: the homophyily score $\varepsilon(\cdot)$ of positive pairs and the bi-level kernel $\kappa_B$. Notably, the smoothness is negatively correlated with $\varepsilon(\cdot)$. This indicates that the excessively expanded community structures can lead to over-smoothing, making node representations indistinguishable. This issue can be addressed by incorporating graph-level structural information as additional details, which complements community information to enhance the node representations.
\begin{equation} \label{equ:Z*}
    \mathbf{Z}^* = \sigma(\mathbf{X}\mathbf{W}_G + 1/K \sum\nolimits_{k=1}^K\widetilde{\mathbf{A}}^k\mathbf{X}\mathbf{W}_G)
\end{equation}
where $K$ denotes capturing local information from the $K$-hops neighborhood of $\widetilde{\mathbf{A}}$ and $\widetilde{\mathbf{A}} = \widetilde{\mathbf{D}}^{-\frac{1}{2}} (\mathbf{A} + \mathbf{I}) \widetilde{\mathbf{D}}^{-\frac{1}{2}}$.

Note that the standard GNN in Eq.~\ref{equ:gnn} can be viewed as a model that tightly couples linear feature transformation with message passing ($\mathbf{Z} = \widetilde{\mathbf{A}}\mathbf{XW}_G$), while the message passing in Eq.~\ref{equ:Z*} occurs in the post-processing phase of GCL model training. This means that our method adopts a decoupled paradigm for GNN. Specifically, we first use a linear layer and incorporate community information into this linear transformation process ($\mathbf{V}=\mathbf{XW}_G$) via our dual-kernel contrastive loss. Then, a training-free graph convolution operator is performed ($\mathbf{Z}=\widetilde{\mathbf{A}}^k\mathbf{V}$). This decoupled paradigm reduces the training burden of GCL and retains the powerful graph-level information processing ability of GNN.

\subsubsection{Graph Representational Similarity Distillation.}
 We adopt a knowledge distillation technique to avoid the significant computational overhead incurred by GNN during inference. Instead of using soft labels as in most previous works~\cite{dis-graph1,t2-gnn}, we directly use the node representations after message passing as the distillation target to encourage the MLP to learn structural information:
\begin{equation} \label{eq:distillation}
    \mathcal{L}_{D} = || \mathrm{MLP}(\mathbf{XW}_G) -  1/K \sum\nolimits_{k=1}^K\widetilde{\mathbf{A}}^k\mathbf{X}\mathbf{W}_G)||_F^2
\end{equation}
Thus, Eq.~\ref{equ:Z*} can be rewritten as:
\begin{equation} \label{eq:Z*D}
    \mathbf{Z}^* = \sigma(\mathbf{X}\mathbf{W}_G + \mathrm{MLP}(\mathbf{XW}_G))
\end{equation}

 Notably, we use the node-level features output by the GCCL as input to the distillation model. Thus, the MLP can capture both community structure and positional information of $G$, which has been shown to be beneficial for graph representational similarity distillation~\cite{nosmog}.
\section{Theoretical Analysis}
We provide theoretical evidence to support the effectiveness of our model, with detailed proofs available in Appendix A. 
\subsection{Properties of Dual-Kernel GCCL Loss}
First, we show that the dual-kernel GCCL loss can approximate the graph contrastive loss on a $k$-step graph diffusion matrix of Eq.~\ref{equ:gcl}. 
\begin{proposition} \label{proposition:appro}
    Assuming the original features $\mathbf{X}$ and the mapped features $\phi(\mathbf{X})$ are bounded by $S_\mathbf{X} :=\max_i||\mathbf{X}_i||_2$ and $S_{\phi(\mathbf{V})} :=\max_i||\phi(\mathbf{V}_i)||^2_2$, respectively. Then, the original contrastive loss of the $k$-step diffusion graph $\mathbf{A}^k$, denoted as $\mathcal{L}_G$, can be approximated by the dual-kernel community contrastive loss, $\mathcal{L}_\mathcal{P}^{lc}$, without considering the influence of combination coefficients:
    \begin{equation*}
        |\mathcal{L}_G-\mathcal{L}^{lc}_\mathcal{P}| \leq L|| \mathbf{A}^k - \mathbf{PP}^\mathrm{T} ||_F S_\mathbf{X}||\mathbf{W}_P||_2 + S_{\phi(\mathbf{V})}
    \end{equation*}
\end{proposition}


Proposition~\ref{proposition:appro} shows that our method can capture the high-order structural information of multi-hop neighborhoods. Minimizing $||\mathbf{PP}^\mathrm{T}-\mathbf{A}||$ is equivalent to minimizing edges between nodes in different communitys, which is a classic minimum cut problem in graph theory~\cite{cut}. This can be achieved by graph partition algorithms, as these algorithms inherently maximize the sum of degrees within communities relative to their external degrees~\cite{structcomp}. Next, we establish formal guarantees for the learned graph representations on downstream tasks.  

\begin{proposition}~\label{proposition:erroe}
    Let $G$ be a graph with $B$ classes and the classes are balanced. Then, there exists a linear function $g(\cdot): \mathcal{X}_G \to \mathbb{R}^C$ such that the error upper bound is
    \begin{equation}
        \mathbb{E}_{v}[|| y(v) - g(\mathbf{v}) ||_2^2] \leq 1+B^2\sum\nolimits_{v} (1+\mathcal{L}_\mathcal{P}(v) - \varepsilon(v)) 
    \end{equation}
\end{proposition}

Proposition~\ref{proposition:erroe} shows that the classification error on learned representations is bounded by the dual-kernel contrastive loss $\mathcal{L}_\mathcal{P}$ and the one-hop homophily score $\varepsilon(v)$ of node $v$ in $\mathbf{\Bar{A}}$. Note that $\varepsilon(v)$ is affected by the graph partition. In general, overly expansive community structures tend to result in a low value of $\varepsilon(v)$. Combining with Proposition~\ref{proposition:soomth}, this requires introducing appropriate graph-level structural information to ensure performance. Conversely, in heterophilic graphs, an expanding receptive field provides additional information that cannot be captured within local neighborhoods~\cite{smooth}. This means we can adapt to graphs with different homophily levels by adjusting the number of communities and the range of local neighborhoods.



\subsubsection{Remark.} The contrastive loss on coarsened graph can be seen as a special case of our method, i.e., when $\alpha=0$ in Eq.~\ref{eq:lc}. Consequently, our method naturally inherits the properties of these method. For instance, our dual-kernel GCCL loss can be seen as introducing an additional regularization term with better generalization, which makes our method more robust to minor perturbation~\cite{structcomp}. Please refer to the Appendix C.1 for more details.

\subsection{Properties of Distillation Loss}
 Based on the graph homophily assumption, nodes of the same semantic class typically share similar neighborhood representations. Thus, the local neighborhood representation $\mathbf{Z}$ can be viewed as sampled from a standard Gaussian distribution centered at $\mathbf{Z}_Y$, i.e., $Z|Y \sim N(\mathbf{Z}_Y, I)$, where $Y$ denotes the latent semantic class of the $K$-hop patterns and $Z$ is the random variable corresponding to $\mathbf{Z}$~\cite{graphacl}. Then, following~\cite{eccv}, we have:
\begin{proposition} \label{proposition:dis}
    Minimizing the distillation loss $\mathcal{L}_D$ is equivalent to maximizing mutual information between the representation $\mathbf{V}$ and the $K$-hop pattern $Y$:
    \begin{equation}
        \mathcal{L}_D \geq H(V|Y) - H(V) = -I(Y;V)
    \end{equation}
    where $V$ is the random variable corresponding to $\mathbf{V}$.
\end{proposition}

Proposition~\ref{proposition:dis} shows that minimizing the distillation loss in Eq.~\ref{eq:distillation} can promote the maximizing of mutual information $I(Y;V)$ between node representations containing community information and the latent semantic classes of the $K$-hop patterns. This allows the distillation model to simultaneously leverage both community and graph-level structural information. Given the above characteristics, the distilled representations exhibit performance comparable to, or even better than, the pre-distillation ones. We will verify this conclusion in the following experiments.

\begin{table*}[t!]  
	\centering
		\begin{tabular}{lcccccccc}
			\toprule[1pt]
			Methods & Cora & CiteSeer & PubMed & Wiki-CS  & Amz.Photo & Co.CS & Co.Physics \\
                \midrule[0.5pt]
                DGI & 82.12$\pm$1.28 & 71.58$\pm$1.21 & 78.87$\pm$2.64 & 75.73$\pm$0.13  & 91.49$\pm$0.25 & 91.95$\pm$0.40 & 94.57$\pm$0.39 \\
			GCA & 79.04$\pm$1.39 & 65.62$\pm$2.46 & 81.55$\pm$2.47 & 79.35$\pm$0.42 & 92.78$\pm$0.17 & 93.32$\pm$0.12 & 95.87$\pm$0.15 \\
                gCooL & 81.63$\pm$1.39 & 71.32$\pm$1.64 & \underline{82.16$\pm$1.31} & 78.87$\pm$0.22 & 93.18$\pm$0.12 & 93.27$\pm$0.15 & 95.13$\pm$0.11 \\
			CSGCL & 79.39$\pm$1.57 & 70.03$\pm$1.49 & 80.37$\pm$2.06 & 78.57$\pm$0.14 & 93.24$\pm$0.37 & 93.59$\pm$0.09 & 95.32$\pm$0.24 \\
                SP-GCL & 82.78$\pm$1.35 & 71.81$\pm$1.06 & 81.14$\pm$1.82 & 80.21$\pm$0.37 & 92.49$\pm$0.31 & 93.05$\pm$0.10 & 95.12$\pm$0.15 \\
                GraphECL & \underline{82.88$\pm$0.95} & \underline{72.26$\pm$0.89} & 82.14$\pm$1.63 & 80.17$\pm$0.15 & 93.39$\pm$0.46 & \underline{94.12$\pm$0.16} & \underline{96.03$\pm$0.07} \\
			SGRL & 82.64$\pm$1.92 & 71.73$\pm$1.58 & 80.91$\pm$1.84 & 80.67$\pm$0.26 & 93.29$\pm$0.42 & 93.61$\pm$0.26 & 95.99$\pm$0.10 \\        
			\midrule[0.5pt]            
			BGRL & 82.33$\pm$1.35 & 71.59$\pm$1.42 & 79.23$\pm$1.74 & 78.74$\pm$0.22 & 93.24$\pm$0.29 & 93.26$\pm$0.36 & 95.76$\pm$0.38 \\
			SUGRL & 81.34$\pm$1.23 & 71.02$\pm$1.77 & 80.53$\pm$1.62 & 79.12$\pm$0.67 & 93.07$\pm$0.15 & 92.83$\pm$0.23 & 95.38$\pm$0.11 \\				
			GGD & 82.34$\pm$1.57 & 71.04$\pm$1.47 & 81.28$\pm$1.31 & 78.72$\pm$0.61  & 92.53$\pm$0.63 & 92.44$\pm$0.19 & 95.03$\pm$0.21 \\
			SGCL & 82.57$\pm$1.43 & 71.65$\pm$1.31 & 81.93$\pm$1.66 & 79.85$\pm$0.53 & \underline{93.46$\pm$0.31} & 93.29$\pm$0.17 & 95.78$\pm$0.11 \\
			St.Comp & 81.28$\pm$1.29 & 71.46$\pm$1.54 & 80.47$\pm$1.63 & 80.57$\pm$0.11  & 92.62$\pm$0.14 & 92.56$\pm$0.12 & 95.44$\pm$0.10 \\						
			E2Neg & 81.47$\pm$1.67 & 71.69$\pm$1.92 & 80.93$\pm$1.49 & \underline{81.12$\pm$0.57} & 93.36$\pm$0.76 & 93.48$\pm$0.59 & 95.86$\pm$0.29 \\
			\midrule[0.5pt]											
			Ours & \textbf{83.77$\pm$1.37} & \textbf{72.68$\pm$1.19} & \textbf{82.56$\pm$1.85} & \textbf{81.75$\pm$0.36}  & \textbf{93.86$\pm$0.15} & \textbf{94.68$\pm$0.14} & \textbf{96.12$\pm$0.17} \\	
		\bottomrule[1pt]
        \toprule[1pt]
			Methods & Cornell & Texas & Wisconsin & Actor  & Crocodile & Amz.Ratings & Questions \\
			\midrule[0.5pt]
                HGRL & 51.78$\pm$1.03 & 61.83$\pm$0.71 & 63.90$\pm$0.58 & 27.95$\pm$0.30 & 61.87$\pm$0.45 & 38.37$\pm$0.36 & - \\
			L-GCL & 52.11$\pm$2.37 & 60.68$\pm$1.18 & 65.28$\pm$0.52 & 32.55$\pm$1.18 & 60.18$\pm$0.43 & - & - \\
			DSSL & 53.15$\pm$1.28 & 62.11$\pm$1.53 & 62.25$\pm$0.55 & 28.15$\pm$0.31 & 62.98$\pm$0.51 & - & - \\
			SP-GCL & 52.29$\pm$1.21 & 59.81$\pm$1.33 & 60.12$\pm$0.39 & 28.94$\pm$0.69 & 61.72$\pm$0.21 & 43.11$\pm$0.32 & 75.08$\pm$0.49 \\
			GREET & 72.91$\pm$1.13 & \underline{84.59$\pm$4.20} & 80.98$\pm$5.62 & 36.14$\pm$1.38 & \underline{66.75$\pm$0.56} & 41.19$\pm$0.25 & - \\
			GraphACL & 59.33$\pm$1.48 & 71.08$\pm$0.34 & 69.22$\pm$0.40 & 30.03$\pm$0.13 & 66.17$\pm$0.24  & 41.49$\pm$0.45 & 74.85$\pm$0.98 \\
                HeterGCL & 75.48$\pm$2.83 & 74.71$\pm$3.59 & 75.58$\pm$4.47 & \underline{37.20$\pm$0.44} & 65.42$\pm$0.57 & - & - \\	
			PolyGCL & 73.78$\pm$3.51 & 72.16$\pm$3.51 & 76.08$\pm$3.33 & 34.37$\pm$0.69 & 65.95$\pm$0.59 & \underline{44.29$\pm$0.43} & \underline{75.33$\pm$0.67} \\
                M3P-GCL & \underline{75.59$\pm$3.81} & 80.84$\pm$1.62 & \underline{81.67$\pm$2.23} & 35.12$\pm$0.97 & 65.67$\pm$0.31  & 42.91$\pm$0.17 & - \\
			\midrule[0.5pt]											
			Ours & \textbf{76.49$\pm$2.43} & \textbf{85.41$\pm$3.01} & \textbf{85.17$\pm$3.02} & \textbf{37.74$\pm$0.78} & \textbf{67.05$\pm$0.72}   & \textbf{47.51$\pm$0.68} & \textbf{76.35$\pm$1.05} \\
            \bottomrule[1pt]                
        \end{tabular}
        \caption{Node classification results on homophilic (top) and heterophilic (bottom) graphs ($\%$). AUC is used for Questions and accuracy for all other datasets. Best and second-best results are shown in \textbf{bold} and \underline{underline}, respectively. '-' indicates either unavailability of the official implementation or exceeding 24 GB GPU memory during evaluation.}
	\label{tab:nc}
\end{table*}

\section{Experimental Study}

\subsection{Experimental Setup}
\subsubsection{Datasets.}
We evaluate on 16 datasets, including 7 homophilic graphs, 7 heterophilic graphs, and 2 large-scale graphs. See Appendix D.1 for dataset statistics.

\subsubsection{Baselines.}
We compare our model with the following three categories of methods, which are described in detail in Appendix D.2.
\begin{itemize}
    \item Classic GCL methods: DGI~\cite{dgi}, GCA~\cite{gca}, gCooL~\cite{gcool}, CSGCL~\cite{csgcl}, SP-GCL~\cite{sp-gcl}, GraphECL~\cite{graphecl} and SGRL~\cite{sgrl}. 
    \item Efficiency-oriented GCL methods: BGRL~\cite{bgrl}, SUGRL~\cite{sugrl}, GGD~\cite{ggd}, SGCL~\cite{sgcl}, StructComp~\cite{structcomp} and E2Neg~\cite{e2neg}.
    \item Heterophily-aware GCL methods: HGRL~\cite{hgrl}, L-GCL~\cite{lgcl}, DSSL~\cite{dssl}, GREET~\cite{greet}, GraphACL~\cite{graphacl}, HeterGCL~\cite{hetergcl}, PolyGCL~\cite{polygcl} and M3P-GCL~\cite{m3pgcl}.
\end{itemize}

\subsubsection{Evaluation Protocols.}
We evaluate downstream task performance using a linear classifier trained on frozen graph representations. Results are averaged over 10 random splits per dataset, with standard deviations reported.



\subsubsection{Implementation Details.}
A two-layer MLP is used as the distillation model. All experiments are implemented using PyTorch and run on a server equipped with an NVIDIA 3090 GPU (24GB memory). The detailed hyperparameter settings are reported in Appendix D.3.

\subsection{Experimental Results}
\subsubsection{Exp-1: Effectiveness Evaluation.}
We conducted comprehensive node classification experiments on both homophilic and heterophilic graphs to evaluate the effectiveness of our method, as shown in Table~\ref{tab:nc}. We also present supplementary results on other graph analysis tasks in Appendix D.4, demonstrating that our method can effectively adapt to various downstream tasks. 

These results demonstrate that: (i) Our method exhibits consistent and superior generalization performance across graphs with varying levels of homophily. (ii) Community-based methods, such as gCooL, CSGCL and E2Neg, show significant competitiveness in node classification, confirming the effectiveness of leveraging community structure in GCL. (iii) StructComp performs contrastive learning on coarsened graphs, ignoring node-level information, which may result in suboptimal performance on node-level tasks.

\subsubsection{Exp-2: Scalability Evaluation.}
We evaluate the scalability of our method by comparing it with efficiency-oriented GCL methods on large-scale graphs, as shown in Table~\ref{tab:nc-large} and Figure~\ref{fig-inference}. For fairness, we excluded memory footprint reports for methods trained with mini-batch processing. 

These results demonstrate that: (i) Our method consistently achieves the best performance on large-scale graphs. (ii) Although StructComp achieves the lowest training overhead by ignoring node-level information, our method effectively improves performance with only a slight increase in computational complexity, demonstrating a well-balanced trade-off between scalability and performance. It is noteworthy that the accuracy gain over SOTA baselines is 1.7$\%$ on Ogbn-Products with two million nodes, which is a substantial improvement given that these methods have been carefully fine-tuned on the corresponding datasets. (iii) Our method consistently outperforms other methods in inference efficiency, and the inference time scales linearly with graph size. On Ogbn-Products, our method is about 180$\times$ faster than the best baseline, which highlights the superiority of knowledge distillation technique for the decoupled GNN.

\begin{figure}[t]
\centering
\includegraphics[width=0.44\textwidth]{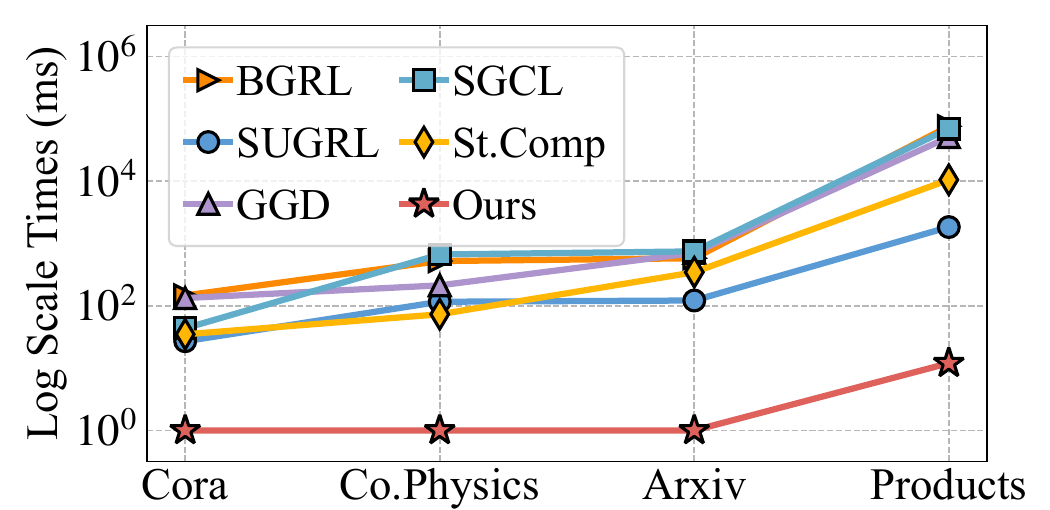} 
\caption{Inference efficiency comparison.}
\label{fig-inference}
\end{figure}

\begin{table*} %
	\centering
		\begin{tabular}{lccccc|ccccc}
			\toprule[1pt]
			\multirow{2}{*}{Methods} &\multicolumn{5}{c}{Ogbn-Arxiv} &\multicolumn{5}{c}{Ogbn-Products}  \\
			\cmidrule[0.5pt](lr){2-11} 
			&Acc &Time.T(\textbf{s}) &Mem.T &Time.I(\textbf{s}) &Mem.I &Acc &Time.T(\textbf{m}) &Mem.T  &Time.I(\textbf{s}) &Mem.I \\
                \midrule[0.5pt]
			BGRL &71.6 & 1.43 &10.7 &0.58 &6.1 &64.0 & 53.3  & - &76.33 &22.8 		\\
			SUGRL &67.8 &0.11 &\textbf{2.6} &\underline{0.12} &\underline{1.5} &72.9 &1.5 & 23.5 &\underline{1.84} &21.3		\\
			GGD &71.6 &0.95 &14.3 &0.71 &1.9 &75.7 & 12.7 & - &143.36 & 22.8 		\\
			SGCL &71.0 & 0.09 & 5.1  &0.75 &4.2 &\underline{76.0} & 1.9 & - &69.94 &22.9		\\
			St.Comp &\underline{71.8} &\textbf{0.05} &\underline{3.4} &0.35 &1.6 &75.5 &\textbf{0.001} &\textbf{5.3} &10.54 &\underline{12.0}		\\
			\midrule[0.5pt]
			Ours &\textbf{72.2} &\underline{0.08} & 4.2 &\textbf{0.001} &\textbf{1.1} &\textbf{77.7} &\underline{0.003} & \underline{8.8} &\textbf{0.01} &\textbf{5.3}		\\
			\bottomrule[1pt]
		\end{tabular}
        	\caption{Scalability evaluation on large-scale datasets. Time.T / I: training / inference time per epoch; Mem.T / I: peak GPU memory usage during training / inference (GB).}    
	\label{tab:nc-large}
\end{table*}

\begin{figure*}[t]
\centering
\includegraphics[width=0.98\textwidth]{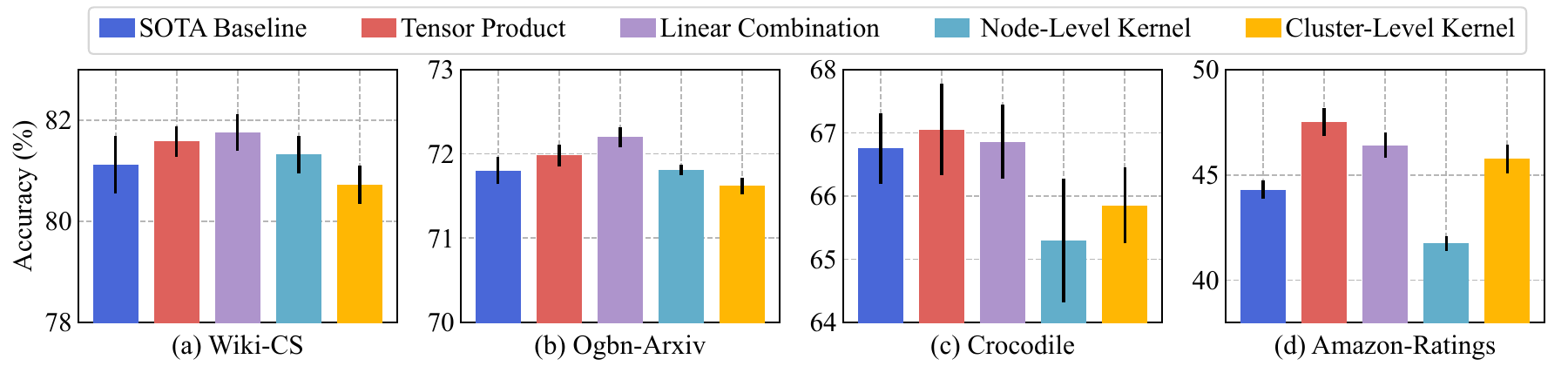} 
\caption{Comparison of different dual-kernel graph community contrastive loss variants.}
\label{fig-kernel}
\end{figure*}


\subsubsection{Exp-3: Necessity of Dual-Kernel.}
We validate the necessity of using dual-kernel contrastive loss to integrate node-level and community-level information by comparing different variants of our method with SOTA baselines. Specifically, we analyze two schemes, the tensor product and linear combination schemes, as well as two additional variants that focus solely on node-level kernel $(\alpha=0)$ and community-level kernel $(\alpha=1)$, as shown in Figure~\ref{fig-kernel}. 

These results demonstrate that: (i) The dual-kernel method outperforms SOTA baselines and the variants focusing on a single kernel, which highlights the effectiveness of MKL in integrating diverse levels of information. (ii) On homophilic graphs, the linear combination scheme performs better, with the node-level kernel outperforming the community-level kernel. Conversely, on heterophilic graphs, the tensor product scheme and the community-level kernel achieve superior performance. This suggests that homophilic graphs should pay more attention to node-level information. Furthermore, we analyze the impact of the combination coefficient $\alpha$ on performance in Appendix D.4.

\section{Related Work}
\subsubsection{Graph Contrastive Learning with Community.}
Recent studies have demonstrated the effectiveness of exploiting community structure in GCL, which can be categorized into view-optimized and loss-optimized methods. (i) View-optimized methods focus on preserving community information in augmented views. For instance, SEGA~\cite{sega} uses an encoding tree containing hierarchical community properties as the anchor view. CI-GCL~\cite{ci-gcl} constrains the view augmentation process based on community invariance. StructComp~\cite{structcomp} performs GCL on community-wise graph. (ii) Loss-optimized methods can effectively avoid mislabeling closely connected nodes as negative samples. For example, gCooL~\cite{gcool} considers nodes and the centroid of their respective communities as positive sample pairs. CS-GCL~\cite{csgcl} adjusts the weight of contrastive samples based on community strength. E2Neg~\cite{e2neg} selects representative negative samples from communities. Despite these significant advancements, they are still limited by the message passing mechanism of GNN.

\subsubsection{Kernel-Based Representation Learning.}
Kernel methods have been used to address the scalability issues of Graph Transformers in supervised learning scenarios, as they can bypass the cumbersome explicit computation of all-pairs attentions~\cite{sgformer,polynormer}. MKL further enhances the expressiveness of kernel methods by combining multiple kernel functions to integrate features from different perspectives~\cite{mkl-two,mkl-tkde}. Examples include recent studies that leverage MKL for federated learning~\cite{mkl-federated}, clustering~\cite{mkl-cluster} and graph classification~\cite{n2c}. Despite the widespread application of kernel and MKL methods, their utilization for unsupervised graph representation learning remains an underexplored area.


\section{Conclusion}
In this work, we propose a scalable and efficient dual-kernel graph community contrastive learning method, underpinned by a straightforward graph partition algorithm and MKL techniques. This design enables us to capture community-level structural features in linear time while preserving essential node-level information. Furthermore, the proposed knowledge distillation technique of the decoupled GNN is particularly suitable for latency-constrained applications. Both theoretical analysis and experimental evaluations verify the effectiveness of our method. We also envision future directions, such as learning adaptive graph partition, integrating edge-based attribute features, and extending to dynamic graphs or more complex graph applications. 

\section{Acknowledgments}
This work was supported by Yunnan Fundamental Research Project (202501AS070102), Program of Yunnan Key Laboratory of Intelligent Systems and Computing (202405AV340009), Future Industry Science and Technology Special Project of Yunnan University (YDWLCY202505), and Scientific Research Fund Project of Yunnan Education Department (2025Y0061). For any correspondence, please refer to Liang Duan.

\bibliography{aaai2026}

\clearpage

\section{A.~Detailed Proofs}
\subsection{A.1~Proof of Proposition 1}
\begin{proposition} \label{proposition:drop-p}
    Let the feature dimension of the community-level feature space $\mathcal{X}_P$ be $d^P$. Then, the expected number of distinct partitioned substructures generated by the Dropout($\cdot$) operation for each partition $P_j$ is:
    \begin{equation}
        \mathbb{E}[|P^s_{j}|s=1,\cdots,d^P|] = d^P\left( 1-(1-p)^{|P_j|}\right)
    \end{equation}
    where $P^s_{j}$ is a substructure of $P_j$ on the feature dimension $s$, and $p$ is the dropout probability.
\end{proposition}

\begin{proof}
    We consider each dimension of community-level features as an aggregation of one-dimensional node features within a partitioned substructure. Let $C^s_j$ denote the indicator random variable for $P_j^s$:
    \begin{equation}
        C^s_j = \begin{cases} 
            0, & \text{if $P^s_{j} = P_{j}$} \\
            1, & \text{otherwise}
        \end{cases}
    \end{equation}
    For the partitioned substructure $P^s_{j}$ to be identical to the original partition $P_{j}$, all nodes within the partition must aggregate their one-dimensional features towards the community centroid. Given that the probability of any node in the partitioned substructure aggregating its features to the community centroid is $1-p$, then in this scenario, we have:
    \begin{equation}
        P(C^s_j) = \begin{cases} 
            (1-p)^{|P_j|}, & C^s_j = 0 \\
            1 - (1-p)^{|P_j|}, & C^s_j = 1
        \end{cases}    
    \end{equation}
    Then, the expected value of $C^s_j$ is 
    \begin{equation}
        \begin{split}
            \mathbb{E}[C^s_j] & = 0 \cdot P(C^s_j=0) + 1 \cdot P(C^s_j=1) \\
                        & = 1 - (1-p)^{|P_j|}
        \end{split}
    \end{equation}
    According to the linearity of expectation, we have:
    \begin{equation}
        \begin{split}
            \mathbb{E}[|P^s_{j}|s=1,\cdots,d^P|] &= \mathbb{E} [\sum\nolimits_k^{d^P} C^s_j] = \sum\nolimits_k^{d^P}{\mathbb{E}[C^s_j]} \\
            & = d^P\left( 1-(1-p)^{|P_j|}\right)
        \end{split}
    \end{equation}
    To this end, we can deduce Proposition 1.
\end{proof}

\subsection{A.2~Proof of Proposition 2}
\begin{proposition} \label{proposition:soomth}
    Let $\Bar{\mathbf{A}}$ be the normalized adjacency matrix constructed from positive node pairs in the contrastive loss $\mathcal{L}_{\mathcal{P}}$ and $y(v)$ denote the label of $v$. Then, the bound of the smoothness between node embeddings is:
    \begin{equation}
    \left\|\mathbf{V} - \Bar{\mathbf{A}}\mathbf{V} \right\|_F \leq \sqrt{2}L \sum_{v_i \in \mathcal{V}}  \frac{1}{1 + \frac{\varepsilon(v_i)\lambda_{v_i}}{(1-\varepsilon(v_i))\gamma_{v_i}}}
    \end{equation}
    where $\lambda_{v_i} =  \mathbb{E}_{v_t \in \mathcal{N}(v_i), y(v_i) = y(v_t)} \kappa_{B}( \{\mathbf{v}_i, \mathbf{c}_j\}, \{\mathbf{v}_t, \mathbf{c}_k \} )$ and $\gamma_{v_i} = \mathbb{E}_{v_t \in \mathcal{N}(v_i), y(v_i)\neq y(v_t)} \kappa_{B}( \{\mathbf{v}_i, \mathbf{c}_j\}, \{\mathbf{v}_t, \mathbf{c}_k \} )$. $L$ is the Lipschitz constant, and $\varepsilon(v_i)$ is the one-hop homophily score of node $v_i$ in $\Bar{\mathbf{A}}$, defined as:
    \begin{equation}
        \varepsilon(v_i) = \frac{1}{|\mathcal{N}(v_i)|} \sum\nolimits_{v_t \in \mathcal{N}(v_i)} \mathds{1}[y(v_i)=y(v_t)] 
    \end{equation}
    where $\mathds{1}[\cdot]$ is the indicator function.
\end{proposition}

\begin{proof}
    Let the positive sample score $a_{v_iv_t}$ in the contrastive loss serve as the weight between node $v_i$ and node $v_t$ in the adjacency matrix $\Bar{\mathbf{A}}$, we have:
    \begin{equation}
        \begin{split}
            \left\|\mathbf{V} - \Bar{\mathbf{A}}\mathbf{V} \right\|_F &\leq \sum_{v_i \in \mathcal{V}} || \mathbf{v}_i - \sum_{v_t \in \mathcal{N}(v_i)} a_{v_iv_t}\mathbf{v}_t||_2 \\
        \end{split}
    \end{equation}
    Since $\Bar{\mathbf{A}}$ is row-normalized, we have:
    \begin{equation}
        \begin{split}
           \left\|\mathbf{V} - \Bar{\mathbf{A}}\mathbf{V} \right\|_F &\leq \sum_{v_i \in \mathcal{V}} || \sum_{v_t \in \mathcal{N}(v_i)} (\mathbf{v}_i -  a_{v_iv _t}\mathbf{v}_t)||_2 \\
           & \leq \sum_{v_i \in \mathcal{V}} \sum_{v_t \in \mathcal{N}(v_i)} a_{v_iv_t} \left\|\mathbf{v}_i - \mathbf{v}_t \right\|_2
        \end{split}
    \end{equation}
    Assume that the linear mapping function from node features to labels is L-Lipschitz continuous, we have:
    \begin{equation}
        \begin{split}
            &\left\|\mathbf{V} - \Bar{\mathbf{A}}\mathbf{V} \right\|_F \leq L\sum_{v_i \in \mathcal{V}} \sum_{v_t \in \mathcal{N}(v_i)} a_{v_iv_t} \left\|y(v_i) - y(v_t) \right\|_2 \\
            & = L\sum_{v_i \in \mathcal{V}} \sum_{v_t \in \mathcal{N}(v_i), y(v_i) \neq y(v_t)} a_{v_iv_t} \left\|y(v_i) - y(v_t) \right\|_2 \\
            & = \sqrt{2}L\sum_{v_i \in \mathcal{V}} \sum_{v_t \in \mathcal{N}(v_i), y(v_i) \neq y(v_t)} a_{v_iv_t} \\
            & = \sqrt{2}L\sum_{v_i \in \mathcal{V}} \frac{\sum_{v_t \in \mathcal{N}(v_i), y(v_i) \neq y(v_t)} \kappa_{B}( \{\mathbf{v}_i, \mathbf{c}_j\}, \{\mathbf{v}_t, \mathbf{c}_k \} ) }{ \sum_{v_t \in \mathcal{N}(v_i)} \kappa_{B}( \{\mathbf{v}_i, \mathbf{c}_j\}, \{\mathbf{v}_t, \mathbf{c}_k \} )} \\
        \end{split}
    \end{equation}
    Let $\lambda_{v_i} =  \mathbb{E}_{v_t \in \mathcal{N}(v_i), y(v_i) = y(v_t)} \kappa_{B}( \{\mathbf{v}_i, \mathbf{c}_j\}, \{\mathbf{v}_t, \mathbf{c}_k \} )$, $\gamma_{v_i} = \mathbb{E}_{v_t \in \mathcal{N}(v_i), y(v_i)\neq y(v_t)} \kappa_{B}( \{\mathbf{v}_i, \mathbf{c}_j\}, \{\mathbf{v}_t, \mathbf{c}_k \} )$, and $\varepsilon(v_i)$ is the one-hop homophily score of node $v_i$. Then, we can obtain:
    \begin{equation}
        \begin{split}
            &\left\|\mathbf{V} - \Bar{\mathbf{A}}\mathbf{V} \right\|_F \\ 
            &\leq \sqrt{2}L\sum_{v_i \in \mathcal{V}} \frac{|\mathcal{N}(v_i)|(1-\varepsilon(v_i))\gamma_{v_i}}{|\mathcal{N}(v_i)|(1-\varepsilon(v_i))\gamma_{v_i} + |\mathcal{N}(v_i)|\varepsilon(v_i)\lambda_{v_i}} \\
            & = \sqrt{2}L \sum_{v_i \in \mathcal{V}}  \frac{1}{1 + \frac{\varepsilon(v_i)\lambda_{v_i}}{(1-\varepsilon(v_i))\gamma_{v_i}}}
        \end{split}
    \end{equation}
    Here, we complete the proof of Proposition~\ref{proposition:soomth}.
\end{proof}

\subsection{A.3~Proof of Proposition 3}
To prove Proposition ~\ref{proposition:appro}, we first introduce a lemma that shows the contrastive loss on the original graph is close to the sum of the coarsened contrastive loss and the low-rank approximation gap~\cite{structcomp}.
\begin{lemma}
    Assuming the original features $\mathbf{X}$ is bounded by $S_\mathbf{X} :=\max_i||\mathbf{X}_i||_2$. Then, the contrastive loss of the $k$-step diffusion graph $\mathbf{A}^k$, denoted as $\mathcal{L}_G$, can be approximated by the coarsened contrastive loss $\mathcal{L}_{S}$ of StructComp.
    \begin{equation}
        |\mathcal{L}_G-\mathcal{L}_{S}| \leq L|| \mathbf{A}^k - \mathbf{PP}^\mathrm{T} ||_F S_\mathbf{X}||\mathbf{W}_P||_2
    \end{equation}
\end{lemma}
Intuitively, this lemma shows that the community-level kernel can approximate the contrastive loss of the $k$-step diffusion graph $\mathbf{A}^k$.
\begin{proposition} \label{proposition:appro}
    Assuming the original features $\mathbf{X}$ and the mapped features $\phi(\mathbf{X})$ are bounded by $S_\mathbf{X} :=\max_i||\mathbf{X}_i||_2$ and $S_{\phi(\mathbf{V})} :=\max_i||\phi(\mathbf{V}_i)||^2_2$, respectively. Then, the original contrastive loss of the $k$-step diffusion graph $\mathbf{A}^k$, denoted as $\mathcal{L}_G$, can be approximated by the dual-kernel community contrastive loss, $\mathcal{L}_\mathcal{P}^{lc}$, without considering the influence of combination coefficients:
    \begin{equation*}
        |\mathcal{L}_G-\mathcal{L}^{lc}_\mathcal{P}| \leq L|| \mathbf{A}^k - \mathbf{PP}^\mathrm{T} ||_F S_\mathbf{X}||\mathbf{W}_P||_2 + S_{\phi(\mathbf{V})}
    \end{equation*}
\end{proposition}

\begin{proof}
    According to the triangle inequality for absolute values, we have:
    \begin{equation}
        \begin{split}
            |\mathcal{L}_G-\mathcal{L}^{lc}_\mathcal{P}| &= |\mathcal{L}_G - \mathcal{L}_{S} + \mathcal{L}_{S} -\mathcal{L}^{lc}_\mathcal{P}| \\
            & \leq|\mathcal{L}_G-\mathcal{L}_{S}| + |\mathcal{L}_S-\mathcal{L}^{lc}_\mathcal{P}|
        \end{split}
    \end{equation}
    We denote the node-level and community-level positive pairs in the dual-kernel GCCL loss with the linear combination method $\ell_{lc}(v_i, P_j)$ as $\ell^+_{lc}(v_i)$ and $\ell^+_{lc}(P_j)$, respectively. Then, for the term $|\mathcal{L}_S-\mathcal{L}^{lc}_\mathcal{P}|$, we can obtain:
    \begin{equation}
        \begin{split}
            |\mathcal{L}_S-\mathcal{L}^{lc}_\mathcal{P}| &= |\frac{1}{n}\sum_{v_i \in \mathcal{V}}\ell^+_{lc}(P_j) - \frac{1}{n}\sum_{v_i \in \mathcal{V}}\left(\ell^+_{lc}(P_j)+ \ell^+_{lc}(v_i)\right)| \\
            & = |\frac{1}{n}\sum_{v_i \in \mathcal{V}}\left(\ell^+_{lc}(v_i)\right)| \\
            & = |\frac{1}{n}\sum_{v_i \in \mathcal{V}}\sum_{P_k \in \mathcal{N}(P_j)}\sum_{v_t\in P_k} \phi(\mathbf{v}_i)^\mathrm{T}\phi(\mathbf{v}_t)| \\
            & \leq |\frac{1}{n}\sum_{v_i \in \mathcal{V}}\sum_{v_t \in \mathcal{V}} \phi(\mathbf{v}_i)^\mathrm{T}\phi(\mathbf{v}_t) | \\
        \end{split}
    \end{equation}
    Let $\phi(\mathbf{v}_{max})$ be the upper bound of feature map $\phi(\mathbf{v}_{i})$, where $i \in n$, we can obtain:
    \begin{equation}
        \begin{split}
            |\mathcal{L}_S-\mathcal{L}^{lc}_\mathcal{P}| &\leq |\frac{1}{n}\sum_{v_i \in \mathcal{V}}\sum_{v_t \in \mathcal{V}} \phi(\mathbf{v}_i)^\mathrm{T}\phi(\mathbf{v}_t) | \\
            &= \frac{1}{n}\sum_{v_i \in \mathcal{V}}\sum_{v_t \in \mathcal{V}} | \phi(\mathbf{v}_i)^\mathrm{T}\phi(\mathbf{v}_t)| \\
            &\leq \frac{1}{n}\sum_{v_i \in \mathcal{V}}\sum_{v_t \in \mathcal{V}}|\phi(\mathbf{v}_{max})^\mathrm{T}\phi(\mathbf{v}_{max})| \\
            &= \frac{1}{n}\sum_{v_i \in \mathcal{V}}\sum_{v_t \in \mathcal{V}} || \phi(\mathbf{v}_{max}) ||_2^2 \\
            &\stackrel{\text{c}}{=} S_{\phi(\mathbf{V})}
        \end{split}
    \end{equation}
    where $\stackrel{\text{c}}{=}$ denotes that minimizing $\frac{1}{n}\sum_{v_i \in \mathcal{V}}\sum_{v_t \in \mathcal{V}} || \phi(\mathbf{v}_{max}) ||_2^2$ is equivalent to minimizing $S_{\phi(\mathbf{V})}$. 
    
    Combining Eq.~28, Eq.~29 and Eq.~31, we can derive:
    \begin{equation*}
        |\mathcal{L}_G-\mathcal{L}^{lc}_\mathcal{P}| \leq L|| \mathbf{A}^k - \mathbf{PP}^\mathrm{T} ||_F S_\mathbf{X}||\mathbf{W}_P||_2 + S_{\phi(\mathbf{V})}
    \end{equation*}

\end{proof}

\subsection{A.4~Proof of Proposition 4}
\begin{proposition}~\label{proposition:erroe}
    Let $G$ be a graph with $B$ classes and the classes are balanced. Then, there exists a linear function $g(\cdot): \mathcal{X}_G \to \mathbb{R}^B$ such that the error upper bound is
    \begin{equation}
        \mathbb{E}_{v}[|| y(v) - g(\mathbf{v}) ||_2^2] \leq 1+B^2\sum\nolimits_{v} (1+\mathcal{L}_\mathcal{P}(v) - \varepsilon(v)) 
    \end{equation}
\end{proposition}

\begin{proof}
 Let the one-hot label corresponding to $y(v) \in \mathbb{R}^{1\times B}$ be $\mathbf{Y}_v$, and assume there exists a linear mapping matrix $\mathbf{W} \in \mathbb{R}^{d^G\times B}$ that maps node features to labels. For any class $i$, the number of nodes in that class is $b_i = \frac{n}{B}$, since we assume an ideal class-balanced setting. Then we have:
    \begin{equation}
        \begin{split}
            \mathbb{E}_{v}[|| y(v) &- g(\mathbf{v}) ||_2^2] = \frac{1}{n} || \mathbf{Y} - \mathbf{VW}||_2^2 \\
            & = \frac{1}{n} || \mathbf{Y} - \Bar{\mathbf{A}}\mathbf{B}  + \Bar{\mathbf{A}}\mathbf{B} - \mathbf{VW}||_2^2 \\
            & \leq \frac{1}{n} || \mathbf{Y} - \Bar{\mathbf{A}}\mathbf{B} ||_2^2 + \frac{1}{n} || \Bar{\mathbf{A}}\mathbf{B} - \mathbf{VW} ||_2^2
        \end{split}
    \end{equation}
    where $\mathbf{B}_{v,i}=b^{-1}_i\mathds{1}[y_v = i]$. For the first term $\frac{1}{n} || \mathbf{Y} - \Bar{\mathbf{A}}\mathbf{B} ||_2^2$, we can obtain:
    \begin{equation}
        \begin{split}
            \frac{1}{n} || \mathbf{Y} &- \Bar{\mathbf{A}}\mathbf{B} ||_2^2 = \frac{1}{n} || \sum_{v \in \mathcal{V}} \left( \mathbf{Y}_v - \Bar{\mathbf{A}}\mathbf{B}  \right) ||_2^2 \\
            &\leq \frac{1}{n} \sum_{v \in \mathcal{V}} || \mathbf{Y}_v - \Bar{\mathbf{A}}\mathbf{B} ||_2^2 \\
            & = \frac{1}{n} \sum_{v \in \mathcal{V}} \sum_{i=1}^B\left( \mathbf{Y}_{v,i} - (\Bar{\mathbf{A}}\mathbf{B})_{v,i} \right)^2 \\
            & = \frac{1}{n} \sum_{v \in \mathcal{V}} [(1 - \sum_{u \in \mathcal{N}(v)} \frac{1}{b_i} \mathds{1}[y_v = i])^2  \\ 
            & ~~~~~~~~~~~~ + ( \sum_{u \in \mathcal{N}(v)} \sum_{i=1}^B \mathds{1}[y_v \neq i] \frac{1}{b_i} \mathds{1}[y_u = i] ) ^2 ] \\
        \end{split}
    \end{equation}
    By simplifying the above formula, we can obtain:
    \begin{equation}
        \begin{split}
             \frac{1}{n} || \mathbf{Y} &- \Bar{\mathbf{A}}\mathbf{B} ||_2^2  \leq  \frac{1}{n} \sum_{v \in \mathcal{V}} [ 1 + \frac{B^2}{n^2} ( \sum_{u \in \mathcal{N}(v)} \mathds{1}[y_v \neq y_u]  )^2] \\
             &\leq  \frac{1}{n} \sum_{v \in \mathcal{V}} [ 1 + \frac{B^2}{n^2} \sum_{u \in \mathcal{N}(v)} \mathds{1}[y_v \neq y_u]] \\
             &\leq \frac{1}{n} \sum_{v \in \mathcal{V}} [ 1 + \frac{B^2}{n^2} \frac{n}{|\mathcal{N}(v)|}\sum_{u \in \mathcal{N}(v)} \mathds{1}[y_v \neq y_u]] \\
             & \leq \frac{1}{n} \sum_{v \in \mathcal{V}} [ 1 + \frac{B^2}{n} (1-\varepsilon(v))]
        \end{split}
    \end{equation}
    For the second term $\frac{1}{n} || \Bar{\mathbf{A}}\mathbf{B} - \mathbf{VW} ||_2^2$ in Eq.~33, we have:
    \begin{equation}
        \begin{split}
            \frac{1}{n} &|| \Bar{\mathbf{A}}\mathbf{B} - \mathbf{VW} ||_2^2 = \frac{1}{n} || \Bar{\mathbf{A}}\mathbf{B} - \mathbf{VV}^\mathrm{T}\mathbf{B} + \mathbf{VV}^\mathrm{T}\mathbf{B} -  \mathbf{VW} ||_2^2 \\
            &= \frac{1}{n} || (\Bar{\mathbf{A}} - \mathbf{VV}^\mathrm{T})\mathbf{B} + \mathbf{V}(\mathbf{V}^\mathrm{T}\mathbf{B} -  \mathbf{W})  ||_2^2 \\
            & \leq \frac{1}{n} || (\Bar{\mathbf{A}} - \mathbf{VV}^\mathrm{T})\mathbf{B} ||_2^2 + \frac{1}{n} || \mathbf{V}(\mathbf{V}^\mathrm{T}\mathbf{B} -  \mathbf{W})  ||_2^2
        \end{split}
    \end{equation}
    Since the $i$-th row of $\mathbf{B}$ is the average representation of nodes in class $i$, we have:
    \begin{equation}
        \begin{split}
            \frac{1}{n} || \Bar{\mathbf{A}}\mathbf{B} - \mathbf{VW} ||_2^2  &\leq \frac{1}{n} || (\Bar{\mathbf{A}} - \mathbf{VV}^\mathrm{T})\mathbf{B} ||_2^2 \\
            &\leq \frac{1}{n} || (\Bar{\mathbf{A}} - \mathbf{VV}^\mathrm{T})||_2^2 ||\mathbf{B}||_2^2 \\
            &= \frac{1}{n} \frac{B^2}{n^2}|| (\Bar{\mathbf{A}} - \mathbf{VV}^\mathrm{T})||_2^2
        \end{split}
    \end{equation}
    Recent work has shown that finding the global optimal value of the contrastive loss is equivalent to solving a matrix factorization problem, i.e., $\min ||\Bar{\mathbf{A}} - \mathbf{VV}^\mathrm{T}||_2^2$~\cite{matrix-factor}. Therefore, combining Proposition 3, our dual-kernel graph community contrastive loss can approximate the contrastive loss of the original graph, and we have:
    \begin{equation}
        \begin{split}
            \frac{1}{n} || \Bar{\mathbf{A}}\mathbf{B} - \mathbf{VW} ||_2^2  &\leq \frac{1}{n} \frac{B^2}{n^2}|| (\Bar{\mathbf{A}} - \mathbf{VV}^\mathrm{T})||_2^2 \\
            &\stackrel{\text{c}}{=} \frac{1}{n} \frac{B^2}{n^2} \mathcal{L}_\mathcal{P}(v)
        \end{split}
    \end{equation}
    Then, combining Eq.~33, Eq.~35 and Eq.~38, we can derive:
    \begin{equation}
        \begin{split}
            \mathbb{E}_{v}[|| & y(v) - g(\mathbf{v}) ||_2^2] \leq \\
            & \frac{1}{n} \sum_{v \in \mathcal{V}} [ 1 + \frac{B^2}{n} (1-\varepsilon(v))] + \frac{1}{n} \frac{B^2}{n^2} \mathcal{L}_\mathcal{P} \\
            &\leq 1 + \sum_{v \in \mathcal{V}} [B^2(1-\varepsilon(v) + B^2 \mathcal{L}_\mathcal{P}(v)] \\
            &= 1+B^2\sum_{v \in \mathcal{V}} (1+\mathcal{L}_\mathcal{P}(v) - \varepsilon(v)) 
        \end{split}
    \end{equation}
        To this end, we complete the proof of Proposition 4.
\end{proof}

\subsection{A.5~Proof of Proposition 5}
\begin{proposition} \label{proposition:dis}
    Minimizing the distillation loss $\mathcal{L}_D$ is equivalent to maximizing the mutual information between the representation $\mathbf{V}$ and the $K$-hop pattern $Y$:
    \begin{equation}
        \mathcal{L}_D \geq H(V|Y) - H(V) = -I(Y;V)
    \end{equation}
    where $V$ is the random variable corresponding to $\mathbf{V}$.
\end{proposition}

\begin{proof}
Let the distilled node representation $\mathbf{V}=\text{MLP}(\mathbf{XW}_G)$ and the local neighborhood representation $\mathbf{Z}=1/K \sum\nolimits_{k=1}^K\widetilde{\mathbf{A}}^k\mathbf{X}\mathbf{W}_G$. Based on the graph homophily assumption, nodes of the same semantic class typically share similar neighborhood representations. Thus, the local neighborhood representation $\mathbf{Z}$ can be viewed as sampled from a standard Gaussian distribution centered at $\mathbf{Z}_Y$, i.e., $Z|Y \sim N(\mathbf{Z}_Y, I)$, where $Y$ denotes the latent semantic class of the $K$-hop patterns and $Z$ is the random variable corresponding to $\mathbf{Z}$. Then, following~\cite{eccv}, we can interpret $\mathcal{L}_D$ as the conditional cross-entropy between $V$ and $Z$, given the pseudo labels $Y$ under the $K$-hop pattern:
\begin{equation}
    \begin{split}
         \mathcal{L}_D &= \frac{1}{n} \sum_{v \in \mathcal{V}} || \mathbf{V}_v - \mathbf{Z}_v ||_2^2 \\
         &\stackrel{\text{c}}{=} H(V;Z|Y)\\
         &=H(V|Y) + \mathcal{D}_{KL}(V||Z|Y) \\
         &\geq H(V|Y) 
    \end{split}
\end{equation}
where $ \mathcal{D}_{KL}(\cdot|\cdot)$ is the KL divergence. The above equality holds because the KL divergence is non-negative. According to the definition of mutual information, we have:
\begin{equation}
    \begin{split}
        \mathcal{L}_D &\geq H(V|Y) \\
        &\geq H(V|Y) - H(V) \\
        &= -I(Y;V)
    \end{split}
\end{equation}
The above equality holds because the entropy $H(\cdot)$ is non-negative. Thus, we complete the proof of Proposition 5.
\end{proof}

\section{B.~Additional Explanations for GCCL}
\subsection{B.1~Two Variants of GCCL Loss}
We consider two kernel functions $\kappa_G$ and $\kappa_P$ defined on the node-level space $\mathcal{X}_G$ and community-level space $\mathcal{X}_P$, respectively. Let the corresponding kernel matrices be
\begin{equation}
    \kappa_G(\mathbf{v}_i, \mathbf{v}_t) = <\phi_G(\mathbf{v}_i), \phi_G(\mathbf{v}_t)>, ~~\mathbf{v}_i,\mathbf{v}_t \in \kappa_G
\end{equation}
and
\begin{equation}
    \kappa_P(\mathbf{c}_J, \mathbf{c}_k) = <\phi_P(\mathbf{c}_j), \phi_P(\mathbf{c}_k)>, ~~\mathbf{c}_j,\mathbf{c}_k \in \kappa_P.
\end{equation}

\subsubsection{Tensor Product Method.} We now consider one variant of the GCCL loss based on the tensor product of kernels:
\begin{equation}
    \begin{split}
        \kappa_B&(\{\mathbf{v}_i, \mathbf{c}_j\}, \{\mathbf{v}_t, \mathbf{c}_k \}) = \kappa_G(\mathbf{v}_i, \mathbf{v}_t) \cdot \kappa_P(\mathbf{c}_j, \mathbf{c}_k) \\
        & = <\phi_G(\mathbf{v}_i), \phi_G(\mathbf{v}_t)> \cdot <\phi_P(\mathbf{c}_j), \phi_P(\mathbf{c}_k)> \\
        & = \phi_G(\mathbf{v}_i)^{\mathrm{T}}\phi_G(\mathbf{v}_t) \cdot \phi_P(\mathbf{c}_j)^{\mathrm{T}}\phi_P(\mathbf{c}_k) \\
        & = \left(\phi_G(\mathbf{v}_i)\otimes\phi_P(\mathbf{c}_j)\right)^\mathrm{T} \left(\phi_G(\mathbf{v}_t)\otimes\phi_P(\mathbf{c}_k) \right)
    \end{split}
\end{equation}
where $\otimes$ represents kronecker product. This formulation indicates that the variant of GCCL Loss based on the tensor product of kernels performs a outer product of the node-level and community-level feature maps, and subsequently uses the resulting product for contrastive loss computation. This method enables full-dimensional interactions across different granularity levels, providing a tight integration of node-level and community-level structural information. Empirical results suggest that this variant is particularly beneficial for node-level tasks on hterophilic graphs.

\subsubsection{Linear Combination Method.} We now consider another variant of the GCCL loss based on the linear combination of kernels:
\begin{equation}
    \begin{split}
        \kappa_B&(\{\mathbf{v}_i, \mathbf{c}_j\}, \{\mathbf{v}_t, \mathbf{c}_k \}) = \alpha\kappa_G(\mathbf{v}_i, \mathbf{v}_t) + \beta\kappa_P(\mathbf{c}_j, \mathbf{c}_k) \\
        & = \alpha<\phi_G(\mathbf{v}_i), \phi_G(\mathbf{v}_t)> + \beta <\phi_P(\mathbf{c}_j), \phi_P(\mathbf{c}_k)> \\
        & = \alpha\phi_G(\mathbf{v}_i)^{\mathrm{T}}\phi_G(\mathbf{v}_t) + \beta \phi_P(\mathbf{c}_j)^{\mathrm{T}}\phi_P(\mathbf{c}_k) \\
        & = \left(\sqrt{\alpha}\phi_G(\mathbf{v}_i)\oplus \sqrt{\beta}\phi_P(\mathbf{c}_j)\right)^\mathrm{T} \\ &~~~~~~~~~~~~~~~~~~~~~~~~\left(\sqrt{\alpha}\phi_G(\mathbf{v}_t)\oplus \sqrt{\beta}\phi_P(\mathbf{c}_k) \right)
    \end{split}
\end{equation}
where $\oplus$ represents the weighted concatenation of the node-level and community-level feature maps. This method preserves the independence of features at different levels, allowing the model to flexibly adjust their relative importance. Experimental results demonstrate that this variant is advantageous for node-level tasks on homophilic graphs, as it enables the model to emphasize node-level features by assigning a larger value to the coefficient $\alpha$.

\subsection{B.2~Model Training}
The overall training process of our method is divided into two stages. The first stage trains the GCL model $f_\Omega$ via the dual-kernel contrastive loss, and the second stage trains the distillation model $f_\Phi$ by minimizing the distance between the local representation and the community representation. The training procedure is provided in Algorithm 1.

Given a graph $G$ with $n$ nodes and $m$ communities, suppose the dimension of the node-level feature space $\mathcal{X}_G$ is $d^G$ and the dimension of the community-level feature space $\mathcal{X}_P$ is $d^P$. Then, the complexity of obtaining the bi-level pair of features is $O(nd(d^P+d^G))$. By leveraging the kernel trick to linearize the node-level contrastive loss, the complexity is reduced from $O(n^2d^G)$ to $O(nd^G)$. The computational complexity of the community-level contrastive loss is $O(m^2d^P)$. Since $m \ll n$,  the dual-kernel contrastive loss has a linear complexity with respect to $n$.

\begin{algorithm}[tb]
\caption{Model Training Procedure}
\label{alg:train}
\textbf{Input}: a graph $G=(\mathbf{A},\mathbf{X})$\\
\textbf{Parameter}: number of communities $m$, type of dual-kernel $s$, training epochs of GCL model $T^g$ and distillation model $T^d$\\
\textbf{Output}: final graph representations $\mathbf{Z}^*$\\
\textbf{Steps}:
\begin{algorithmic}[1] 
\STATE Initialize model parameters $\Omega$ and $\Phi$.
\STATE $\mathcal{P}~\gets$ construct the partition of $G$ by Metis.
\STATE $\mathbf{A}^\mathcal{P} \gets$ construct the community-level graph by $\mathbf{P}^\mathrm{T}\mathbf{AP}$.
\STATE // Stage 1: Training the GCL model $f_\Omega$.
\FOR{$i=1$ to $T^g$}
\STATE $\{\mathbf{v}, \mathbf{c}\} \gets$ generate the bi-level features via Eq.~6.
\IF{$s$ = tensor product}
\STATE $\mathcal{L}_\mathcal{P}~\gets$ calculate the contrastive loss via Eq.~8.
\ELSE
\STATE $\mathcal{L}_\mathcal{P}~\gets$ calculate the contrastive loss via Eq.~9.
\ENDIF
\STATE $\Omega ~\gets$ updates GCL model parameters with $\mathcal{L}_\mathcal{P}$.
\ENDFOR
\STATE // Stage 2: Training the distillation model $f_\Phi$.
\FOR{$i=1$ to $T^d$}
\STATE $\mathbf{V}^\prime ~\gets$ generate distilled graph representations.
\STATE $\mathcal{L}_D ~\gets$ calculate the distillation loss via Eq.~13.
\STATE $\Phi ~\gets$ updates distillation model parameters with $\mathcal{L}_\mathcal{D}$.
\ENDFOR
\STATE $\mathbf{Z}^*~\gets$ generate final graph representations via Eq.~14.
\STATE \textbf{return} $\mathbf{Z}^*$
\end{algorithmic}
\end{algorithm}

\begin{figure*}[t]
\centering
\includegraphics[width=0.99\textwidth]{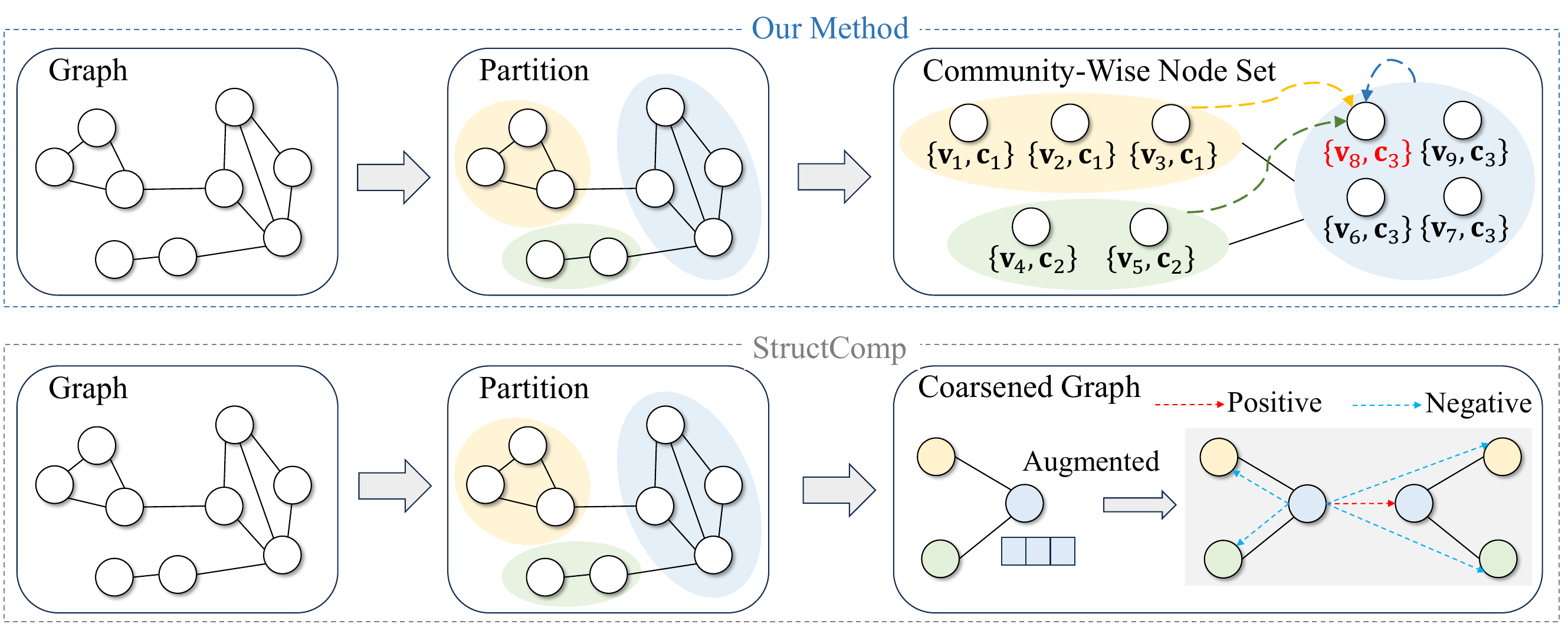} 
\caption{Comparison with StructComp.}
\label{fig-vs-model}
\end{figure*}

\begin{figure*}
\centering
\includegraphics[width=0.99\textwidth]{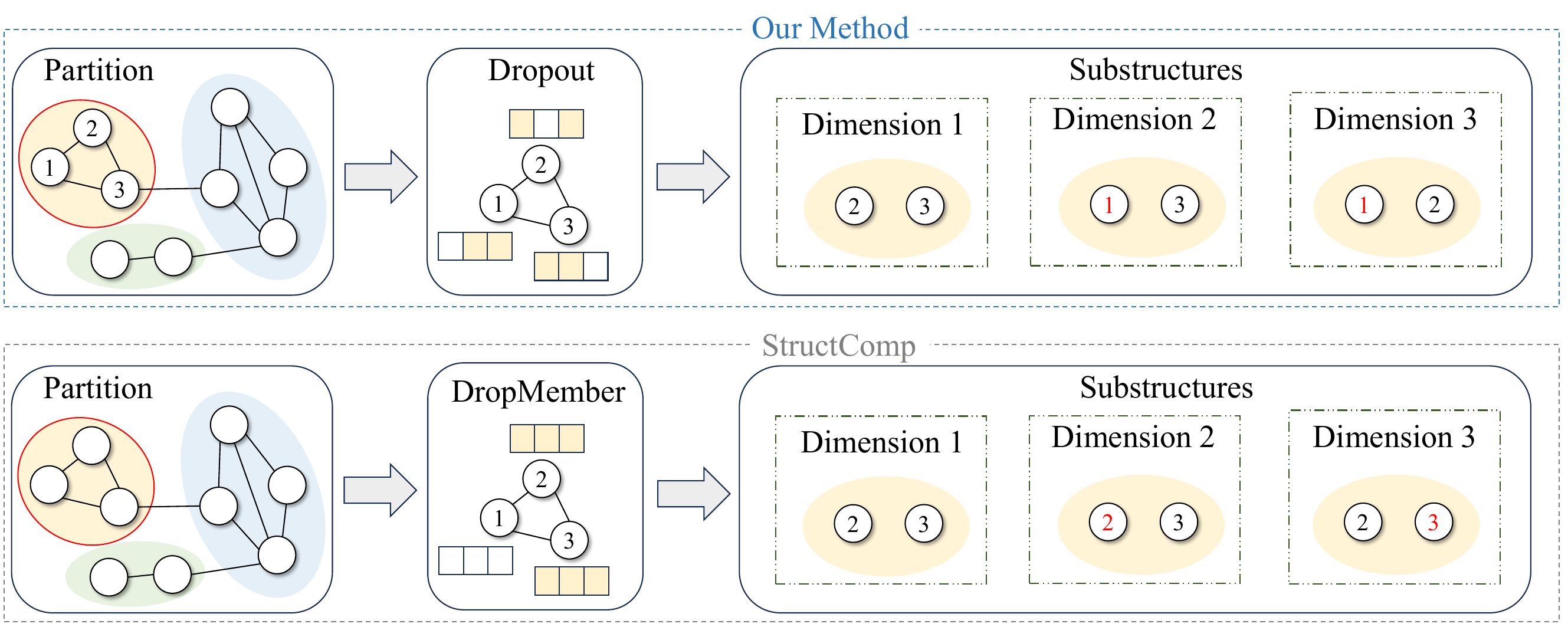} 
\caption{Augmentation strategy comparison with StructComp. Our method can be regarded as constructing a random substructure along each feature dimension, where such diversity of substructures helps enhance generalization ability. In contrast, DropMember directly discards all feature dimensions of certain nodes, which only produces a single substructure.}
\label{fig-vs-aug}
\end{figure*}

\section{C.~Comparison with Related Methods}
\subsection{C.1~Comparison with StructComp}
Our method differs from StructComp~\cite{structcomp} in three key aspects.
\subsubsection{(i) Framework Design.}
StructComp generates a coarsened graph by treating each community as a single node, and then computes the graph contrastive loss on this coarsened graph. In contrast, instead of using coarsening techniques to simplify each community into single node, we envision the graph as a network of node sets interconnected between communities, which allows us to preserve essential node information for model training. Figure~\ref{fig-vs-model} illustrates the primary difference between our method and StructComp in framework design.

\subsubsection{(ii) View Augmentation Strategy.}
StructComp employs a view augmentation method called DropMember, which randomly drops a portion of the nodes within a community to re-aggregate its features. Our method, however, applies a Dropout operation on each node's features, which can be seen as generating a partitioned substructure in each dimension. Figure 5 highlights the main differences between these two augmentation strategies.

\subsubsection{(iii) Contrastive Loss.}
We leverage MKL techniques to compute the graph community contrastive loss. Notably, when we employ a linear combination of kernels with the parameter $\alpha$=0, StructComp becomes a special case of our method, as our approach discards node-level information under this condition, while the community-level kernel function is equivalent to computing the graph contrastive loss on this coarsened graph.

Since StructComp is a special case of our method, our method naturally inherits several of its desirable properties. For example, the contrastive loss on the coarsened graph can be viewed as introducing an additional regularization term to the vanilla InfoNCE loss, which enhances the robustness of the encoder against small perturbations (See Theorem 4.2 in StructComp paper).

\subsection{C.2~Comparison with Other Kernel Methods}
In kernel-based representation learning, $ELU()+1$~\cite{elu} and $ReLU()$~\cite{relu} are two commonly used kernel functions. They perform well in graph classification tasks, where a key feature of such datasets is their extremely small node scale (e.g., only tens of nodes). However, when applied to large-scale graphs with millions of nodes, these kernel functions often lead to training collapse. 

Specifically, in our experiments on Ogbn-Products, replacing the node-level feature map $\phi(v)$ in our graph community contrastive loss with either $ReLU()$ or $ELU()+1$ resulted in NaN training losses. This is because both kernel functions accumulate positive values in the graph representations across the 2 million nodes, causing the denominator of the contrastive loss to exceed the numerical limits of 32-bit floating-point representation, which ultimately leads to NaN values and unstable training. To mitigate this issue, we adopt $Sigmoid()$ as the kernel function of node-level feature map, which maps the graph representations into a bounded range between 0 and 1. This effectively prevents numerical overflow and ensures stable training even on large-scale graphs.

\subsection{C.3~Comparison with N2C-Attn}
Our work is inspired by N2C-Attn~\cite{n2c}. However, there are several key differences between the two approaches. N2C-Attn focuses on supervised graph-level tasks, specifically graph classification based on graph Transformer architectures. It employs Multiple Kernel Learning (MKL) to compute attention scores and outputs community-level representations, which are inherently suited only for graph-level tasks. In contrast, our method is designed for unsupervised node-level tasks, where MKL is used to compute the contrastive loss, while also addressing the inference efficiency bottleneck. Furthermore, N2C-Attn adopts $ReLU()$ and $ELU()+1$ as kernel functions. As discussed in Section C.2, these kernels are unsuitable for large-scale graphs with millions of nodes due to numerical instability and training collapse issues. Therefore, our approach differs fundamentally from N2C-Attn in problem formulation, MKL design, and training objectives.

\begin{table*}[h!]
    \centering
		\begin{tabular}{lccccccc}
			\toprule[1pt]
			Dataset & Nodes & Edges & Classes & Features &Homophily Ratio  & Train / Valid / Test  \\
			\midrule[0.75pt]
			Cora  & 2,708 & 10,556 & 7 & 1,433 & 0.77 & 140 /500 / 1,000 \\
			CiteSeer & 3,327 & 9,104 & 6 & 3,703 & 0.63 & 120 / 500 / 1,000 \\
			Pubmed & 19,717 & 88,648 & 3 & 500 & 0.66 &60 / 500 / 1,000 \\
			Wiki-CS & 11,701 & 431,206 & 10 & 300 &0.57 &1,170 / 1,171 / 9,360 \\
			Amazon-Photo & 7,650 & 238,162 & 8 & 745 & 0.77 & 765 / 765 / 6,120 \\
			Coauthor-CS & 18,333 & 163,788 & 15 & 6,805 &0.76 & 1,833 / 1,834 / 14,666 \\
			Coauthor-Physics	& 34,493 & 495,924 & 5 & 841 & 0.85  & 3,449 / 3,450 / 27,594 \\
                \midrule[0.75pt]
			Cornell  & 183 & 295 & 5 & 1,703 &0.0311 & 87 /59 /37 \\
                Texas  & 183 & 309 & 5 & 1,703 & 0.0013 & 87 /59 / 37 \\
			Wisconsin  & 251 & 499 & 5 & 1,703 &0.0941 & 120 /80 / 51 \\
			Actor  & 7,600 & 29,926 & 5 & 932 &0.0110 & 3,634 /2,432 / 1,520 \\
			Crocodile  & 11,631 & 360,040 & 5 & 2,089  &0.0842 & 6,978 /2,327 / 2,326 \\
			Amazon-Ratings  & 24,492 & 186,100 & 5 & 300 &0.1266  & 12,246 /6,123 / 6,123 \\
			Questions  & 48,921 & 307,080 & 2 & 301 &0.0722 & 24,460 /12,230 / 12,231 \\
                \midrule[0.75pt]
			Ogbn-Arxiv  & 169,343 & 1,166,243 & 40 & 128 &0.416 & 90,941 /29,799 / 48,603 \\
			Ogbn-Products  & 2,449,029 & 61,859,140 & 47 & 100 & 0.459 & 196,615 /39,323 / 2,213,091 \\
			\bottomrule[1pt]
		\end{tabular}
        \caption{The detailed dataset statistics.}
		\label{tab:dataset}
\end{table*}

\begin{table*}[ht]
    \centering
		\begin{tabular}{lcccccccc}
			\toprule[1pt]
			Dataset &Lr &Epoch &Partition Rate &$k-$Hop &$d$ & $\alpha$  &$p$ & $\tau$\\
			\midrule[0.75pt]
                Cora  & 0.005 &15 &0.09 &3 & 1,024 &0.6  & 0.1  & 0.09\\
			CiteSeer & 0.05 &15 &0.07 &3 & 2,048 & 0.5 & 0.15 &0.04\\
			PubMed & 0.0005 & 75 &0.01 & 2 &512 & 0.5 & 0.2 &0.08\\
			    Wiki-CS & 0.0005 & 20 & 0.02 & 3 &1,024 &0.7 &0.15 &0.08\\
			Amazon-Photo & 0.01 & 25 & 0.03 & 5 &1,024 &0.6 &0.15 &0.05\\
			Coauthor-CS & 0.005 & 50 & 0.09 & 1  & 1,024 &0.8 &0.1 &0.08\\
			Coauthor-Physics &0.0005 & 20 & 0.04 & 1 & 2,048 &0.7 &0.3 &0.10\\
                \midrule[0.75pt]
                Cornell  & 0.0005 & 20 & 0.2 & 0 & 8,192 & - &0.1 &0.03\\
                Texas  & 0.0001 & 20 & 0.05 & 0 & 8,192 & - &0.5 &0.04\\
			Wisconsin  & 0.005 & 50 & 0.09 & 0 & 4,096 & - &0.55 &0.06\\
			Actor  & 0.01 & 5 & 0.09 & 0 & 2,048 & - &0.55 &0.03 \\
                Crocodile  & 0.05 & 5 & 0.02 & 0 & 8,192 &- &0.5 &0.09\\
                Amazon-Ratings  & 0.001 & 50 &0.06 &2 & 8,192 & - &0.55 &0.09\\
                Questions  & 0.005 & 10 & 0.007 & 5 & 8,192 & - &0.55 &0.05\\
                \midrule[0.75pt]
			Ogbn-Arxiv  & 0.0005 &25 &0.007 &10 & 800 &0.9 &0.1 &0.03\\
			Ogbn-Products  & 0.001 &25 &0.0001 &10 &128 &0.9 &0.05 & 0.06\\
            		\bottomrule[1pt]
		\end{tabular}
        \caption{Details of the hyper-parameters of our method.}
		\label{tab:parametrs}
\end{table*}
\section{D.~Experimental Study}
\subsection{D.1~Dataset Statistics}
\subsubsection{Datasets.} We evaluate our method on 16 benchmark datasets with different scales and homogeneity levels. including: (i) homophilic graphs: Cora, CiteSeer, PubMed, Wiki-CS, Amazon-Photo, Coauthor-CS, and Coauthor-Physics~\cite{cora-data,wiki-data,amazon-co-data}. (ii) 7 heterophilic graphs: Cornell, Texas, Wisconsin, Actor, Crocodile, Amazon-Ratings, and Questions~\cite{cornell-data,corcodile-data,ratings-data}. (iii) 2 large-scale graphs: Ogbn-Arxiv and Ogbn-Products~\cite{ogbn-data}. The summary statistics of the graphs are shown in Table~\ref{tab:dataset}.

\begin{itemize}
    \item \textbf{Cora}, \textbf{CiteSeer} and \textbf{PubMed} are three citation network datasets where nodes represent papers, edges represent citation relationships between papers, features consist of bag-of-words representations of papers, and labels correspond to the research topics of the papers.
    \item \textbf{Wiki-CS} is a reference network extracted from Wikipedia, where nodes represent articles on computer science, edges represent hyperlinks between articles, features are average bag-of-words embeddings of the corresponding article contexts, and labels are the specific fields of each article.
    \item \textbf{Amazon-Photo} and \textbf{Amazon-Ratings} are two co-purchase networks from Amazon, where nodes represent products, edges represent co-purchase relationships (i.e., two products are frequently bought together), features are bag-of-words representations of product reviews, and labels are product categories.
    \item \textbf{Coauthor-CS} and \textbf{Coauthor-Physics} are two co-author networks extracted from the Microsoft Academic Graph in the KDDCup 2016 challenge, where nodes represent authors, edges represent collaborative relationships, features are bag-of-words representations of paper keywords, and labels are the research fields of the authors.
    \item \textbf{Cornell}, \textbf{Texas} and \textbf{Wisconsin} are three networks of web pages from different computer science departments, where nodes represent web pages, edges represent hyperlinks between web pages, features are bag-of-words representations of pages, and labels are types of web pages.
    \item \textbf{Actor} is an actor co-occurrence network, where nodes represent actors, edges indicate co-occurrence relationships between two actors in the same film, features are extracted from keywords on Wikipedia pages, and labels are the categories of the corresponding actors.
    \item \textbf{Crocodile} is a Wikipedia network, where nodes represent web pages, edges represent hyperlinks between web pages, features are extracted from page keywords, and labels are the daily traffic of the pages.
    \item \textbf{Questions} is based on data from the question-answering website Yandex Q, where nodes represent users, edges indicate that two users answered the same question within a year, features are descriptions of users, and labels are the activity levels of users.
    \item \textbf{Ogbn-Arxiv} and \textbf{Ogbn-Products} are two large-scale datasets. Ogbn-Arxiv is a citation network, where nodes represent papers, edges represent citation relationships between papers, features are extracted from titles and abstracts, and labels correspond to the research topics of the papers. Ogbn-Products is a co-purchase network, where nodes represent products, edges represent co-purchase relationships, features are bag-of-words representations of product reviews, and labels are product categories 
\end{itemize}

\subsubsection{Splitting Strategies.} For the Cora, CiteSeer and PubMed datasets, we randomly select 20 nodes per class for training, 500 nodes for validation, and 1,000 nodes for testing~\cite{dgi}. For the other 4 homophilic datasets, we follow previous works and adopt the public $10\%/10\%/80\%$ training/validation/testing split~\cite{greet}. For the heterophilic and large-scale datasets, we use the standard splits provided by PyTorch Geometric~\cite{sgcl,polygcl}. 

\subsection{D.2~Baselines}
GCL exhibits excellent capability in learning graph representations without task-specific labels, with its core idea being to leverage contrastive loss based on mutual information (MI) maximization to distinguish between positive and negative node pairs, thereby training GNNs.
\begin{itemize}
    \item \textbf{DGI} is a foundational GCL method that maximizes MI between node representations and graph summary.
    \item \textbf{GCA} enhances GCL by incorporating adaptive augmentation based on rich topological and semantic priors.
    \item \textbf{gCooL} utilizes community information to construct positive and negative node pairs required for GCL.
    \item \textbf{CSGCL} adjusts the weight of contrastive samples based on community strength.
    \item \textbf{SP-GCL} exploits the centralized nature of node representation, eliminating the need for graph augmentation.
    \item \textbf{GraphECL} improves inference efficiency based on the coupling model of MLP and GNN.
    \item \textbf{SGRL} enhances the diversity of graph representation through a center-away strategy. 
\end{itemize}

Recent studies improve the scalability by simplifying the steps of view encoding or loss calculation in GCL.
\begin{itemize}
    \item \textbf{BGRL} is a GCL method that learns by predicting alternative augmentations of the input.
    \item \textbf{SUGRL} removes widely used data augmentation and discriminator from previous GCL methods.
    \item \textbf{GGD} adopts a binary cross-entropy loss to distinguish between the two groups of node samples
    \item \textbf{SGCL} utilizes the outputs from two consecutive iterations as positive pairs, eliminating the negative samples.
    \item \textbf{StructComp} performs contrastive learning on the constructed coarsened graph to improve scalability.
    \item \textbf{E2Neg} leverages a small number of representative samples to learn discriminative graph representations. 
\end{itemize}

There are also some methods that explore the potential of GCL on heterophilic graphs.
\begin{itemize}
    \item \textbf{HGRL} learns node representations by preserving original features and capturing informative distant neighbors.
    \item \textbf{L-GCL} samples positive examples from the neighborhood and adopts kernelized loss to reduce training time.
    \item \textbf{DSSL} uses latent variable modeling to decouple different neighborhood contexts without data augmentation.
    \item \textbf{GREET}earns node representations by distinguishing homophilic and heterophilic edges.
    \item \textbf{GraphACL} captures two-hop monophily similarities without relying on homophily assumptions.
    \item \textbf{PolyGCL} leverages polynomial filters to generate low-pass and high-pass spectral augmented views, 
    \item \textbf{M3P-GCL} uses the macro-micro message passing to improve performance on heterophilic graphs.
\end{itemize}

\begin{table*}
        \centering
        \begin{tabular}{lcc|cc|cc|cc|cc|cc}
			\toprule[1pt]
			\multirow{2}{*}{Methods} &\multicolumn{2}{c}{Cora} &\multicolumn{2}{c}{CiteSeer} &\multicolumn{2}{c}{Wiki-CS} &\multicolumn{2}{c}{Amz.Photo} &\multicolumn{2}{c}{Co.CS} &\multicolumn{2}{c}{Co.Physics} \\
			\cmidrule[0.5pt](lr){2-13} 
			&NMI &ARI &NMI &ARI &NMI &ARI &NMI &ARI &NMI &ARI &NMI &ARI\\
            \midrule[0.5pt]
                $K$-Means &8.66 &4.81 &22.45 &20.26 &25.71 &15.02 & 25.77 &14.51 &60.12 & 40.37 & 48.94 & 27.59\\
                gCooL &52.83 &46.15 &40.32 &39.04 &38.24 &\underline{26.88} &56.60 &43.14 &75.32 &62.07 &65.19 &57.81 \\
                CSGCL &43.42 &34.13 & 40.76 & 41.96 &37.17 &12.11 & 58.81 &46.33 &77.12 &63.57 &66.13 &58.29\\
                SP-GCL &28.29 &16.62 &37.67 &36.12 &16.33 &5.81 &28.54 &15.17 &62.37 &44.12 &65.43 &45.97\\
                GraphECL &52.10 &42.39 & 25.29 &22.14 &34.76 &19.44 &49.68 &29.41 &74.37 &61.59 &63.17 &60.22\\
                SGRL &46.94 & 36.84 &43.03 &\underline{43.52} &33.27 & 16.59 &33.65 &17.75 &\underline{77.41} &\underline{65.73} &60.88 &55.70\\
                SUGRL & 56.34 &48.43 &41.97 &42.94 &35.27 &21.86  &\underline{59.62} &\underline{49.77} &76.62 &62.53 &65.69 &60.37 \\
                GREET &\underline{55.18} &\underline{49.71} &\underline{43.13} &42.58 &37.36 &22.21 &52.33 & 37.08 & 75.79 & 62.13 & 66.37 & 63.62\\
                SGCL &54.83 &48.02 &39.66 &39.17  &\underline{39.97} &17.61 &52.76 &38.80 &59.49 &52.31 &\textbf{69.14} & \textbf{68.50}\\
                E2Neg &23.21 &8.63 & 36.09 & 34.69 & 29.65 & 13.69 & 33.75 & 17.43 & 75.23 & 57.52 & 59.15 &44.83\\
                \midrule[0.5pt]
                Ours &\textbf{59.56} &\textbf{51.23} &\textbf{43.81} &\textbf{44.49} &\textbf{40.91} &\textbf{28.93} &\textbf{60.17} &\textbf{50.09} &\textbf{79.66} &\textbf{66.75} &\underline{67.25} &\underline{67.34}\\
			\bottomrule[1pt]
		\end{tabular}
        \caption{Node clustering results measured by NMI $(\%)$ and ARI $(\%)$. }
	\label{tab:cluster}
\end{table*}

\begin{figure}
\centering
\includegraphics[width=0.46\textwidth]{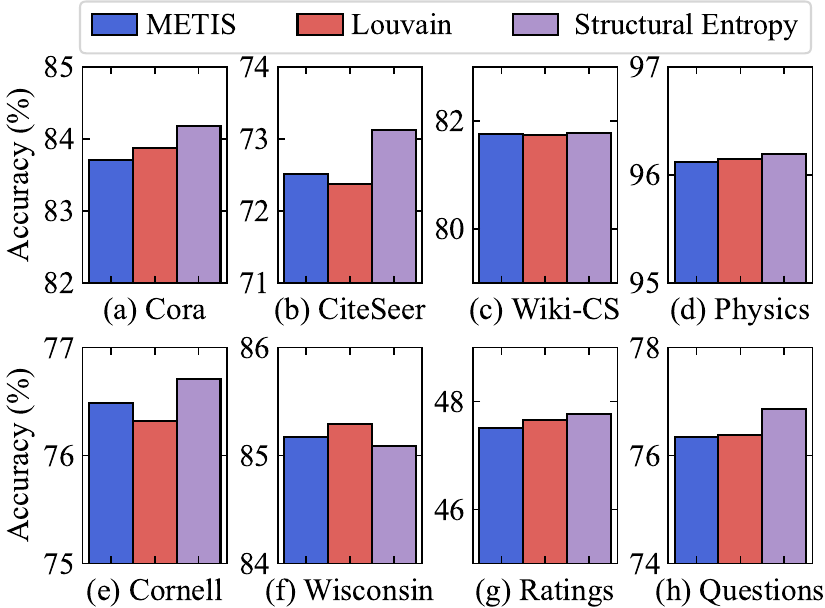} 
\caption{Impacts of partition methods.}
\label{fig-partition}
\end{figure}

\begin{figure*}[t]
\centering
\includegraphics[width=0.99\textwidth]{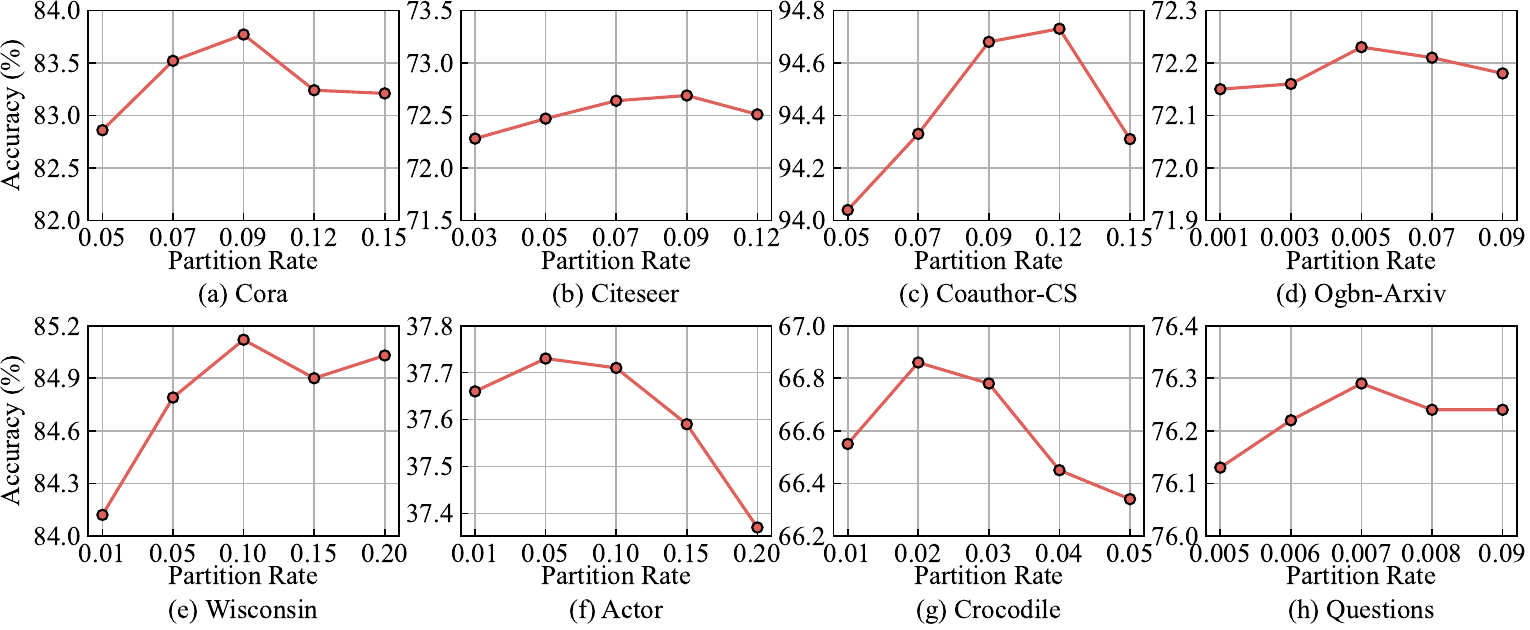} 
\caption{Impact of partition rate.}
\label{fig-rate}
\end{figure*}

\subsection{D.3~Parameter Settings}
We adopt the officially released implementations provided by the authors as baselines and use the hyperparameters specified in their original papers. To ensure fair comparisons, for baselines without reported settings on specific datasets, we perform grid search to carefully tune their hyperparameters. Each dataset is evaluated over 10 different random splits to ensure robustness. All experiments are conducted on a Windows 11 machine equipped with an Intel i9-10900X CPU, 128GB RAM, and an NVIDIA 3090 GPU (24GB memory).

We implement our method in PyTorch with Adam optimizer, with a one-layer linear layer as the encoder and a two-layer MLP as the distillation model. The learning rate $lr$ is selected from ${0.0001, 0.0005, 0.001, 0.002, 0.005, 0.1}$. The number of training epochs is chosen from ${5, 10, 15, 20, 25, 50, 75}$. The partition rate is adjusted based on the number of nodes and classes, and setting the number of communities to tens of times the actual number of classes typically yields optimal results. The order of the diffusion matrix $k$ is selected from ${0, 1, 2, 3, 4, 5}$, and for complex graphs like Ogbn-Arxiv and Ogbn-Products, $k$ is set to 10. The node-level and community-level feature dimensions $d$ are the same, chosen from ${512, 1024, 1500, 2048, 4096, 8192}$. Considering computational cost, we set $d=800$ for Ogbn-Arxiv and $d=128$ for Ogbn-Products. The dropout rate $p$ ranges from 0 to 0.6. The combination coefficient $\alpha$ is selected from the range $[0.5, 1]$, and the temperature coefficient $\tau$ is selected from $[0.01, 0.1]$. The hyperparameters for each dataset are summarized in Table~\ref{tab:parametrs}. More detailed settings can be found in the released code.


\subsection{D.4~Additional Experimental Results}
In this subsection, we provide additional experiments, including node clustering tasks, ablation studies, and analysis of parameter influences.

\subsubsection{Exp-4: Node Clustering.}
We selected several methods that perform well on node classification tasks and compared them in node clustering task, where $K$-Means refers to clustering directly on raw node features. The results are shown in Table~\ref{tab:cluster}.

These results demonstrate that:
(i) Our method outperforms other baselines on most datasets, which can be attributed to its ability to leverage both intra- and inter-community information. (ii) gCooL and CS-GCL also achieve strong performance in the clustering task, further highlighting the importance of community-level information in node representation learning This means that the node representations generated by our method can be extended to other node-level tasks..

\begin{table*}[t!]
    \centering
        \begin{tabular}{lcccccccc}
            \toprule[1pt]
			Variants & Cora & CiteSeer & PubMed & Wiki-CS  & Amz.Photo & Co.CS & Co.Physics \\
			\midrule[0.5pt]
			MLP & 56.11$\pm$0.34 &56.91$\pm$0.42 &71.35$\pm$0.73 &72.02$\pm$0.21 &78.54$\pm$0.05 & 90.42$\pm$0.08 &93.54$\pm$0.05\\
			GCN & 81.60$\pm$1.37 &70.3$\pm$1.15 &79.00$\pm$0.78 &76.87$\pm$0.37 &92.35$\pm$0.25 &93.10$\pm$0.17 &95.54$\pm$0.18\\
			\midrule[0.5pt]
			(w/o Do) & 83.23$\pm$1.37 &\underline{72.23$\pm$1.57} &\underline{82.12$\pm$1.71} & 81.37$\pm$0.25 &93.41$\pm$0.28 &94.28$\pm$0.13 &95.96$\pm$0.12\\			
			(w/o GC) & 74.79$\pm$1.34 & 70.22$\pm$1.56 &75.35$\pm$1.84 & 75.22$\pm$0.49 & 89.33$\pm$0.34 &93.07$\pm$0.21 &95.29$\pm$0.08\\			
			(w/o $\mathcal{L}_{D}$) & \underline{83.58$\pm$1.56} & 71.82$\pm$1.48 & 81.45$\pm$2.36 &\underline{81.42$\pm$0.40} &\underline{93.77$\pm$0.23} &\underline{94.33$\pm$0.14} &\textbf{96.14$\pm$0.16}\\
			\midrule[0.5pt]
			Ours & \textbf{83.77$\pm$1.37} & \textbf{72.68$\pm$1.19} & \textbf{82.56$\pm$1.85} & \textbf{81.75$\pm$0.36}  & \textbf{93.86$\pm$0.15} & \textbf{94.68$\pm$0.14} & \underline{}{96.12$\pm$0.17} \\
            \bottomrule[1pt]
        \toprule[1pt]
            Variants & Cornell & Texas & Wisconsin & Actor  & Crocodile & Amz.Ratings & Questions \\
            \midrule[0.5pt]
            GCN &57.03$\pm$3.30 &60.00$\pm$4.80 &56.47$\pm$6.55 &30.83$\pm$0.77 &\underline{66.72$\pm$1.24} &\textbf{48.70$\pm$0.63} &76.09$\pm$1.27\\
            \midrule[0.5pt]
		(w/o Do) &76.11$\pm$3.07 & \underline{84.59$\pm$4.37} & \underline{84.92$\pm$4.49} &\underline{37.21$\pm$0.78} &66.63$\pm$0.64 &46.95$\pm$0.65 &\underline{76.29$\pm$1.06}\\
            (w/o GC) & - & - & - & - & - & 41.15$\pm$0.61 &70.56$\pm$1.01\\
            (w/o $\mathcal{L}_{D}$) & - & - & - & - & - & 47.02$\pm$0.73 &76.17$\pm$1.08\\
            \midrule[0.5pt]
            Ours & \textbf{76.49$\pm$2.43} & \textbf{85.41$\pm$3.01} & \textbf{85.17$\pm$3.02} & \textbf{37.74$\pm$0.78} & \textbf{67.05$\pm$0.72}   & \underline{47.51$\pm$0.68} & \textbf{76.35$\pm$1.05} \\
        \bottomrule[1pt]                
        \end{tabular}		
    \caption{Ablation study on medium-scale datasets. ‘-’ indicates that we do not use graph convolution operators.}    
    \label{tab:ablation-study-1}
\end{table*}
\subsubsection{Exp-5: Impacts of Graph Partition.}
We analyze the impact of graph partition on performance. First, we we compared several representative partition algorithms, including Louvain~\cite{louvain}, Structural Entropy (SE)~\cite{ais}, and Metis used in our experiments. Then, we investigate the effect of varying the number of communitys. The results are shown in Figures~\ref{fig-partition} and \ref{fig-rate}, respectively.

These results demonstrate that: (i) Our method is compatible with various graph partition algorithms. In general, more advanced algorithms tend to yield better performance (i.e., SE). Considering the complexity of partitioning, we recommend using SE for medium-scale graphs and using the more efficient algorithm Metis for large-scale graphs. (ii) The performance of our method varies with the compression ratio and exhibits a hump-shaped curve. If the number of communitys is too small, excessive compression may degrade performance, while more communitys do not bring better performance. Based on dataset statistics, we find that setting the number of communitys to tens of times the actual number of classes typically yields optimal results.


\subsubsection{Exp-6: Ablation Studies.}
We conducted an ablation study to evaluate the contributions of several key components, as shown in Table~\ref{tab:ablation-study-1}. The specific ablation settings include: (a) Removing the dropout operator (w/o Dropout), (b) Removing the graph convolution operator (w/o GC) and (c) Removing the representation distillation operator (w/o $\mathcal{L}_D$).

These results demonstrate that: (i) All components contribute to the performance of our method. Although the representation distillation module has a relatively minor impact on performance, it is of great value in significantly improving inference efficiency. (ii) Local information is crucial for improving the accuracy of node classification on homophilic graphs, but not always effective on heterophilic graphs. (iii)  Knowledge distillation techniques may have limitations on large-scale graphs. (iv)  Even after removing the GC operator, our method still significantly outperforms a pure MLP, which further proves its effectiveness in capturing high-order structural information.


\begin{figure*}
	\centering
		\includegraphics[width=0.99\textwidth]{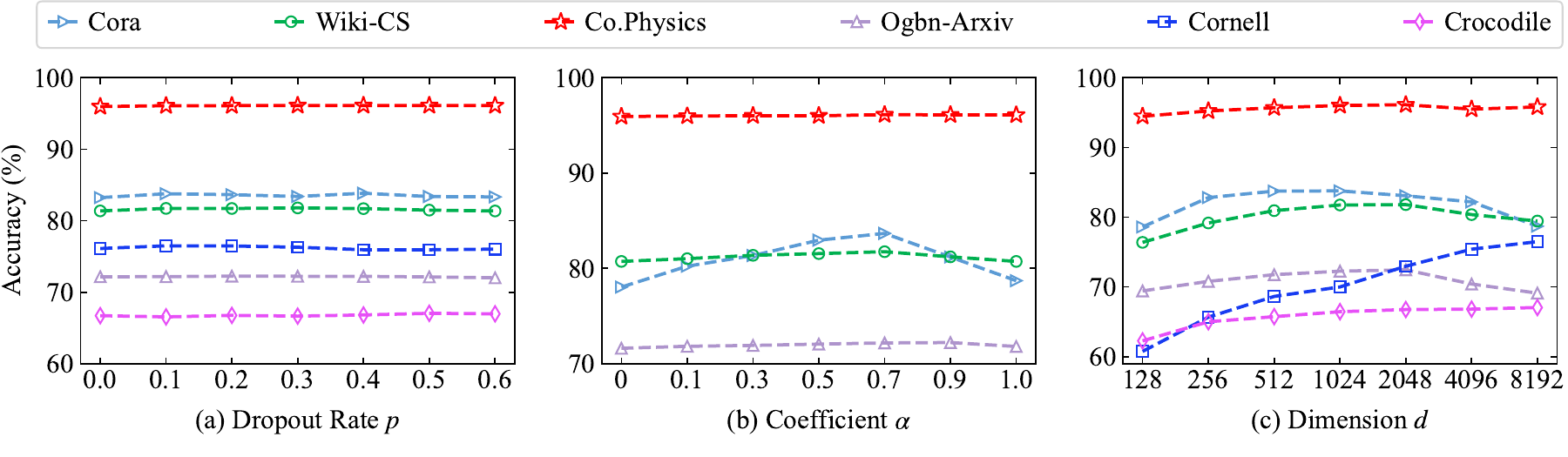}
		\caption{Impacts of dropout rate $p$, combination coefficient $\alpha$ and embedding dimension $d$.}
		\label{fig-hidden}
\end{figure*}

\subsubsection{Exp-7: Sensitivity of Parameters.} 
We investigate the influence of the dropout rate $p$, embedding dimension $d$, and the combination coefficient $\alpha$, as shown in Figure~\ref{fig-hidden}. 

These results demonstrate that: (i) On homophilic graphs, the optimal dropout rate typically falls between 0.1 and 0.3, whereas on terophilic graphs, values of $p$ greater than 0.3 yield better performance. This suggests that promoting substructure diversity is more effective for complex graphs, and such diversity can also reduce the training cycles (as shown in Table 4, our method requires at most 75 training epochs). (ii) A larger embedding dimension $d$ generally improves node classification accuracy, particularly on terophilic graphs. However, on homophilic graphs, extremely large dimensions may lead to overfitting, resulting in a slight performance drop. (iii) On homophilic graphs, the combination coefficient $\alpha$ is typically greater than 0.5, implying that node-level information should be emphasized more heavily for node classification tasks.


\end{document}